%% file: root.tex
\documentclass[letterpaper, 10 pt, journal, twoside]{IEEEtran}

\usepackage{cite}
\usepackage{hyperref}
\usepackage{amsmath,amssymb,amsfonts}
\usepackage{amsthm}
\usepackage{algorithm}
\usepackage{graphicx,import}
\usepackage{textcomp}
\usepackage{xcolor}
\usepackage{pifont}
\newcommand{\cmark}{\ding{51}}%
\newcommand{\xmark}{\ding{55}}%
\usepackage{bm,nicefrac}
\usepackage{algpseudocode}
\usepackage{url}
\def\BibTeX{{\rm B\kern-.05em{\sc i\kern-.025em b}\kern-.08em
    T\kern-.1667em\lower.7ex\hbox{E}\kern-.125emX}}
\usepackage{array}
\newcolumntype{L}[1]{>{\raggedright\let\newline\\\arraybackslash\hspace{0pt}}m{#1}}
\newcolumntype{C}[1]{>{\centering\let\newline\\\arraybackslash\hspace{0pt}}m{#1}}
\newcolumntype{F}[1]{>{\let\newline\\\arraybackslash\hspace{0pt}}m{#1}}
\definecolor{tudelft-fuchsia}{cmyk}{0.19,1,0,0.19}
\definecolor{tudelft-cyan}{cmyk}{1,0,0,0}

\newtheorem{corollary}{Corollary}
\newtheorem{lemma}{Lemma}
\newtheorem{example}{Example}

\ifCLASSINFOpdf
\else
\fi

\begin{document}
\title{Active Inference and Behavior Trees for \\ Reactive Action Planning and Execution in Robotics }

\author{Corrado Pezzato, Carlos Hern\'andez Corbato, Stefan Bonhof, and Martijn Wisse
\thanks{This research was supported by Ahold Delhaize. All content represents the opinion of the authors, which is not necessarily shared or endorsed by their respective employers and/or sponsors}
\thanks{Corrado Pezzato, Carlos Hern\'andez Corbato, Stefan Bonhof, and Martijn Wisse are with the Cognitive Robotics Department,
        TU Delft, 2628 CD Delft, The Netherlands
        {\tt\small \{c.pezzato, c.h.corbato, s.d.bonhof, m.wisse\}@tudelft.nl}} %
}

\markboth{IEEE Transactions on Robotics. Accepted, November 2022}
{Pezzato \MakeLowercase{\textit{et al.}}: Active Inference and Behavior Trees for Reactive Action Planning and Execution in Robotics} 

\maketitle
\input{tex/abstract}


\IEEEpeerreviewmaketitle
\input{tex/introduction}
\input{tex/background}

\input{tex/myIdea}

\input{tex/AImodularizeBT}

\input{tex/experiments}
\input{tex/discussion}

\input{tex/conclusions}


\appendices
\input{tex/appendix}

\ifCLASSOPTIONcaptionsoff
  \newpage
\fi

\bibliography{IEEEabrv,main}
\bibliographystyle{IEEEtran}

\input{tex/bibio}
\end{document}

%% file: tex/abstract.tex
\begin{abstract}
We propose a hybrid combination of active inference and behavior trees (BTs) for reactive action planning and execution in dynamic environments, showing how robotic tasks can be formulated as a free-energy minimization problem. The proposed approach allows handling partially observable initial states and improves the robustness of classical BTs against unexpected contingencies while at the same time reducing the number of nodes in a tree. In this work, we specify the nominal behavior offline, through BTs. However, in contrast to previous approaches, we introduce a new type of leaf node to specify the desired state to be achieved rather than an action to execute. The decision of which action to execute to reach the desired state is performed online through active inference. This results in continual online planning and hierarchical deliberation. By doing so, an agent can follow a predefined offline plan while still keeping the ability to locally adapt and take autonomous decisions at runtime, respecting safety constraints. We provide proof of convergence and robustness analysis, and we validate our method in two different mobile manipulators performing similar tasks, both in a simulated and real retail environment. The results showed improved runtime adaptability with a fraction of the hand-coded nodes compared to classical BTs.  
\end{abstract}

%% file: tex/introduction.tex
\section{Introduction}
\IEEEPARstart{D}{eliberation} and reasoning capabilities for acting are crucial parts of online robot control, especially when operating in dynamic environments to complete long-term tasks. Over the years, researchers developed many task planners with various degrees of optimality, but little attention has been paid to actors \cite{ghallab2014,nau2015}, i.e. algorithms endowed with reasoning and deliberation tools during plan execution. Works as \cite{ghallab2014,nau2015} advocate a change in focus, explaining why this lack of actors could be one of the main causes of the limited spread of automated planning applications. Authors in \cite{Colledanchise2017,Colledanchise2019,safronov2020task} proposed the use of BTs as graphical models for more reactive task execution, showing promising results. Other authors also tried to answer the call for actors \cite{Paxton2019, Garrett2020}, but there are still open challenges to be addressed. These challenges have been identified and highlighted by many researchers, and can be summarized as in \cite{ghallab2014, Colledanchise2019} in two properties that an actor should possess:
\begin{itemize}
    \item \textit{Hierarchical deliberation}: each action in a plan may be a task that an actor may need to further refine online.
    \item \textit{Continual online planning and reasoning}: an actor should monitor, refine, extend, update, change, and repair its plans throughout the acting process, generating activities dynamically at run-time.
\end{itemize}

Actors should not be mere action executors then, but they should be capable of intelligently taking decisions. This is particularly useful for challenging problems such as mobile manipulation in dynamic environments, where actions planned offline are prone to fail. In this paper, we consider mobile manipulation tasks in a retail environment with a partially observable initial state. We propose an actor based on active inference. Such an actor is capable of following a task planned offline while still being able of taking autonomous decisions at run-time to resolve unexpected situations.  

Active inference is a neuroscientific theory that has recently shown its potential in control engineering and robotics \cite{anil_colored_noise, baioumy2020active,pezzato2020active,baioumy2021fault}, particularly in real-world experiments for low-level adaptive control \cite{pezzato2020novel, oliver}. Active inference describes a biologically plausible algorithm for perception, action, planning, and learning. This theory has been initially developed for continuous processes \cite{friston2, buckley, tutorial} where the main idea is that the brain's cognition and motor control functions could be described in terms of \textit{free-energy minimization} \cite{friston1}. In other words, we, as humans, take actions in order to fulfill prior expectations about a \textit{desired prior} sensation \cite{friston3}. Active inference has also been extended to Markov decision processes for discrete decision making  \cite{friston2012active}, recently gathering more and more interest \cite{friston2017active, sajid2021active, schwartenbeck2019computational, Kaplan2018}. In this formulation, active inference is proposed as a unified framework to solve the exploitation-exploration dilemma by acting to minimize the free-energy. Agents can solve complicated problems once provided a \textit{context sensitive prior} about preferences. Probabilistic beliefs about the state of the world are built through Bayesian inference, and a finite horizon plan is selected in order to maximize the evidence for a model that is biased toward the agent’s preferences. At the time of writing, the use of discrete Active Inference for symbolic action planning is limited to low dimensional and simplified simulations \cite{schwartenbeck2019computational, Kaplan2018}. In addition, current solutions rely on fundamental assumptions such as instantaneous actions without preconditions, which do not hold in real robotic situations.

To tackle these limitations of active inference and to address the two main open challenges of hierarchical deliberation and continual planning, we propose a hybrid combination of active inference and BTs. We then apply this new idea to mobile manipulation in a dynamic retail environment. 

\subsection{Related Work}
In this section, we mainly focus on related work on \textit{reactive action planning and execution}, a class of methods that exploits reactive plans, which are stored structures that contain the behavior of an agent. To begin with, BTs \cite{Colledanchise2018, Colledanchise2017} gathered increasing popularity in robotics to specify reactive behaviors. {BTs are de-facto replacing finite state machines (FSM) in several state-of-the-art systems, for instance in the successor of the ROS Navigation Stack, Nav2.  \cite{macenski2020marathon}} BTs are graphical representations for action execution. The general advantage of BTs is that they are modular and can be composed into more complex higher-level behaviors, without the need to specify how different BTs relate to each other. They are also an intuitive representation that modularizes other architectures such as FSM and decision trees, with proven robustness and safety properties \cite{Colledanchise2017}. These advantages and the structure of BTs make them particularly suitable for the class of dynamic problems we are considering in this work, as explained later on in Section \ref{sec:Background}. However, in classical formulations of BTs, the plan reactivity still comes from hard-coded recovery behaviors. This means that highly reactive BTs are usually big and complex and that adding new robotic skills would require revising a large tree. To partially cope with these problems, the authors in \cite{Colledanchise2019} proposed a blended reactive task and action planner which dynamically expands a BT at runtime through back-chaining. The solution can compensate for unanticipated scenarios, but cannot handle partially observable environments and uncertain action outcomes. Conflicts due to contingencies are handled by prioritizing (shifting) sub-trees. Authors in \cite{safronov2020task} extended \cite{Colledanchise2019} and showed how to handle uncertainty in the BT formulation as well as planning with non-deterministic outcomes for actions and conditions. Other researchers combined the advantages of BTs with the theoretical guarantees on the performance of planning with the Planning Domain Definition Language (PDDL) {by representing robot task plans as
robust logical-dynamical systems (RLDS) \cite{Paxton2019}}. RLDS results in a more concise problem description with respect to using BTs only, but online reactivity is again limited to the scenarios planned offline. For unforeseen contingencies, re-planning would be necessary, which is more resource-demanding than reacting as shown in the experimental results in \cite{Paxton2019}. {The output of RLDS is equivalent to a BT, yet BTs remain more intuitive to compose and have more support from the community with open-source libraries and design tools.}

Goal-Oriented Action Planning (GOAP) \cite{Orkin2003}, instead, focuses on online action planning. This technique is used for non-player-characters (NPC) in video games \cite{Orkin2006}. Goals in GOAP do not contain predetermined plans. Instead, GOAP considers atomic behaviors which contain preconditions and effects. The general behavior of an agent can then be specified through very simple Finite State Machines (FSMs) because the transition logic is separated from the states themselves. GOAP generates a plan at run-time by searching in the space of available actions for a sequence that will bring the agent from the starting state to the goal state. However, GOAP requires hand-designed heuristics that are scenario-specific and it is computationally expensive for long-term tasks. 

Hierarchical task Planning in the Now (HPN) \cite{Kaelbling2011} is an alternate plan and execution algorithm where a plan is generated backward starting from the desired goal, using A*. HPN recursively executes actions and re-plans. To cope with the stochasticity of the real world, HPN has been extended to belief HPN (BHPN) \cite{PackKaelbling2013}. A follow-up work \cite{Levihn2013} focused on the reduction of the computational burden by implementing \textit{selective re-planning} to repair local poor choices or exploit new opportunities without the need to re-compute the whole plan which is very costly since the search process is exponential in the length of the plan. 

A different method for generating plans in a top-down manner is the Hierarchical Task Network (HTN) \cite{Erol1994}. At each step, a high-level task is refined into lower-level tasks. In practice \cite{Ghallab2016}, the planner exploits a set of standard operating procedures for accomplishing a given task. The planner decomposes the given task by choosing among the available ones until the chosen primitive is directly applicable to the current state. Tasks are then iteratively replaced with new task networks. An important remark is that reactions to failures, time-outs, and external events are still a challenge for HTN planners. HTN requires the designer to write and debug potentially complex domain-specific recipes and, in very dynamic situations, re-planning might occur too often. 

Finally, active inference is a normative principle underwriting perception, action, planning, decision-making, and learning in biological or artificial agents. Active inference on discrete state spaces \cite{DaCosta2020} is a promising approach to solving the exploitation-exploration dilemma and empowers agents with continuous deliberation capabilities. The application of this theory, however, is still in an early stage for discrete decision making where current works only focus on simplified simulations as proof of concept \cite{friston2017active, Kaplan2018, sajid2021active, schwartenbeck2019computational}. In \cite{Kaplan2018}, for instance, the authors simulated an artificial agent which had to learn and solve a maze given a set of simple possible actions to move (\textit{up, down, left, right, stay}). Actions were assumed instantaneous and always executable. In general, current discrete active inference solutions lack a systematic and task-independent way of specifying prior preferences, which is fundamental to achieving a meaningful behavior, and they never consider action preconditions that are crucial in real-world robotics. As a consequence, plans with conflicting actions which might arise in dynamic environments are never addressed in the current state-of-the-art.

\subsection{Contributions}
In this work, we propose the hybrid combination of BTs and active inference to obtain more reactive actors with hierarchical deliberation and continual online planning capabilities. We introduce a method to include action preconditions and conflict resolution in active inference, as well as a systematic way of providing prior preferences through BTs. The proposed hybrid scheme leverages the advantages of online action selection with active inference and it removes the need for complex predefined fallback
behaviors while providing convergence guarantees. In addition, we provide extensive mathematical derivations, examples, and code, {to understand, test, and reproduce our findings.}

\subsection{Paper structure}
The remainder of the paper is organized as follows. In Section \ref{sec:Background} we provide an extensive background on active inference and BTs. Our novel hybrid approach is presented in Section \ref{sec:Solution}, and its properties in terms of robustness and stability
are analyzed in Section~\ref{sec:Analysis}. In Section~\ref{sec:Experiments} we report the experimental evaluation, showing how our solution can be used with different mobile manipulators for different tasks in a retail environment. Finally, Section~\ref{sec:Discussion} contains the discussion, and Section~\ref{sec:Conclusions} the conclusions. 

%% file: tex/background.tex
\section{Background on Active Inference and BTs}
\label{sec:Background}
\subsection{Background on Active Inference}
Active inference provides a unifying theory for perception, action, decision-making, and learning in biological or artificial agents \cite{DaCosta2020}. {Our active inference agent rests on the tuple $(\mathcal{O},\mathcal{S},\mathcal{A},P,Q)$. This is composed of: a finite set of observations $\mathcal{O}$, a finite set of states $\mathcal{S}$, a finite set of actions $\mathcal{A}$, a generative model $P$ and an approximate posterior $Q$.}

{Active inference proposes a solution for action and perception by assuming that actions will fulfill predictions that are based on inferred states of the world, given some observations. The generative model contains beliefs about future states and action plans, where plans that lead to preferred observations are more likely. Perception and action are achieved through the optimization of two complementary objective functions, the variational free-energy $F$, and the expected free-energy $G$. These quantities to optimize are derived based on the generative model and the approximate posterior, as detailed later.} Variational free-energy measures the fit between the generative model and past and current sensory observations, while expected free-energy scores future possible courses of action according to prior preferences and predicted observations. {Fig.~\ref{fig:genIdea} depicts the general high level idea.} {For a first time reader of active inference, we advise consulting \cite{friston2017active} for a more extensive introduction.}
\begin{figure}[!htb]
    \centering
    \includegraphics[width=0.42\textwidth]{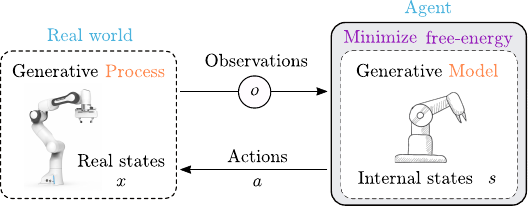}
    \caption{{High-level visualization of an active inference agent. The generative process describes the true causes of the agent's observations $o$. An agent can apply actions to change the state of the world and to get observations that are aligned with its internal preferences.}}
    \label{fig:genIdea}
\end{figure}

In the following, we {explain the form of the generative model and how the model parameters relate to states, actions, and observations. Based on this model, we present the expressions for the free-energy and expected free-energy that are used to derive the equations for perception and decision making. We complement the theory with} {\textit{pen-and-paper} examples and associated Python code\footnote{\url{https://github.com/cpezzato/discrete_active_inference/blob/main/discrete_ai/scripts/paper_examples.py}}.} All the other necessary mathematical derivations are added in the appendices. 

\subsubsection{Generative model}
{In active inference \cite{friston2017active}, the generative model $P$ is chosen to be a Markov process that allows to infer the states of the environment and to predict the effects of actions as well as future observations. This is expressed as a joint probability distribution $P(\bar{o},\bar{s},\bm \eta,\pi)$, where $\bar{o}$ is a sequence of observations, $\bar{s}$ is a sequence of states, $\bm \eta$ represents model parameters, and $\pi$ is a plan. In particular, once given fixed model parameters $\bm \eta$ for a task\footnote{{Note that, in the general case, model parameters are not fixed and can be updated as well through active inference \cite{friston2017active}.}}, we can write:
\begin{equation}
    P(\bar{o},\bar{s},\pi) = P(\pi)\prod_{\tau=1}^T P(s_\tau|s_{\tau-1},\pi)P(o_\tau|s_\tau)
    \label{eq:factorizedModel}
\end{equation}}
The full derivation of the generative model in eq.~\eqref{eq:factorizedModel}, the assumptions under its factorization, and how this joint probability is used to define the free-energy $F$ can be found in Appendix \ref{sec:AppendixP}.
{The probability distributions in eq.~\eqref{eq:factorizedModel} are represented internally by an active inference agent through the following parameters:
\begin{itemize}
    \item $\bm A \in [0, 1]^{r, m}$ is a matrix representation of the conditional probability $P(o_\tau|s_\tau)$, where $r$ is the number of possible observations and $m$ the number of possible states. $\bm A$ is also called the likelihood matrix, and it indicates the probability of observations given a specific state. Each column of $\bm A$ is a categorical distribution. It holds that $P(o_\tau|s_\tau,\bm A)=Cat(\bm As_\tau)$. For a generic entry $\bm A_{ij} = P(o_\tau=i | s_\tau=j)$.
    \item $\bm B \in [0, 1]^{m, m}$ represents a transition matrix. In particular $P(s_{\tau+1}|s_\tau,\ a_\tau)= Cat(\bm B_{a_\tau} s_{\tau})$. For a symbolic action $a_\tau$ in a plan $\pi$, $\bm B_{a_\tau}$ represents the probability of state $s_{\tau+1}$ while applying action $a_\tau$ from state $s_\tau$. The columns of $\bm B_{a_\tau}$ are categorical distributions. 
    \item $\pi$ is a sequence of actions over a time horizon $T$. $\bm \pi$ is the posterior {distribution}, a vector holding the probability of different plans.
    These probabilities depend on the expected free-energy in future time steps under plans given the current belief: $P(\pi) = \sigma(-G(\pi))$. Here, $\sigma$ indicates the softmax function used to normalize probabilities.
\end{itemize}}
{An active inference agents contains also a model $\bm D \in [0,1]^m$ that represents the belief about the initial state at $\tau=1$. So $P(s_0)=Cat(\bm D)$. Additionally, an agent also represents prior preferences about desired observations for goal-directed behavior in $\bm C \in \mathbb{R}^r$, such that $P(o_\tau) = \bm C$.}

In the context of this paper, we do not consider additional generative model parameters such as the $\bm E$ vector to encode priors over plans used to represent habits \cite{smith2021step, hesp2021}. The parameter $\bm E$ could be used to include common sense knowledge in the decision making process, but this will be addressed in future work. Table~\ref{tab:Symbols} summarises the notation adopted in this paper. {The top part contains quantities computed by active inference, while the bottom part the domain parameters required}. 
{As explained next, these internal models will be used by an active inference agent to compute the free-energy and the expected free-energy.}

\begin{table}[ht]
\caption{Notation for Active Inference}
\centering 
\begin{tabular}{C{2.3cm} C{5.7cm}} %
\hline\hline 
\textbf{Symbol} & \textbf{Description}\\ [0.5ex] %
\hline 
{$s_\tau \in \{0, 1\}^m$} & {One-hot encoding of the hidden state at time $\tau$, where $m$ is the number of mutually exclusive states in the discrete state space.} \\

{$\bm s_\tau^\pi \in [0, 1]^m$} & {Posterior distribution over the state under a plan $\pi$, where the elements sum up to one.}\\

{$o_\tau \in \{0, 1\}^r$} & {Observation at time $\tau$ that can have $r$ mutually exclusive possible values.}\\

{$\bm o_\tau^\pi \in [0, 1]^r$} & Posterior distribution of observations under a plan.\\

$\pi$& Plan specifying a sequence of {symbolic} actions $\pi = [a_\tau,\ a_{\tau+1}, ..., a_T]^\top$, where $T$ is the time horizon.\\

{$\bm \pi \in [0,1]^p$}& {Posterior distribution over plans, where $p$ is the number of different plans.}\\

{$F(\pi) \in \mathbb{R}$}& {Plan specific variational free-energy}.\\ 

{$\bm F_\pi \in \mathbb{R}^p$} & {$\bm F_\pi = (F(\pi_1),F(\pi_2),...)^\top$} is a column vector containing the free-energy for every plan.\\ 

{$G(\pi,\tau) \in \mathbb{R}$} & Expected free-energy for a plan at time $\tau$.\\ 

$\bm G_\pi \in \mathbb{R}^p$ & $\bm G_\pi = (G(\pi_1),G(\pi_2),...)^\top$ is a column vector containing the expected free-energy for every plan.\\ 

\hline

{$\bm A \in [0,1]^{r\times m}$} &  Likelihood matrix, mapping from hidden states to
observations $P(o_\tau|s_\tau, \bm A)= Cat(\bm As_\tau).$ \\

{$\bm B_{a_\tau} \in [0,1]^{m\times m}$ }& Transition matrix,  $P(s_{\tau+1}|s_\tau,\ a_\tau)= Cat(\bm B_{a_\tau} s_{\tau})$.\\

{$\bm C \in \mathbb{R}^r$}& Prior preferences over observations $P(o_\tau) = \bm C$.\\

{$\bm D \in [0,1]^m$} & Prior over initial states {$P(s_0)= Cat(\bm D)$}.\\

$\sigma$ & Softmax function.\\

\hline 
\end{tabular}
\label{tab:Symbols}
\end{table}

\subsubsection{Variational Free-energy}
Given the generative model as before, one can derive an expression for the variational free-energy. By minimizing $F$, an agent can determine the most likely hidden states given sensory information. The expression for $F$ is given by:
\begin{eqnarray}
    {F(\pi) = \sum_{\tau=1}^T \bm s_\tau^{\pi\top} \bigg[\ln{\bm s_\tau^\pi} - \ln{{(}\bm B_{a_{\tau-1}}\bm s_{\tau-1}^{\pi}{)}} - \ln{{(}\bm A^\top \bm o_\tau{)}} \bigg]}
    \label{eq:free-energy}
\end{eqnarray}
where $F(\pi)$ is a plan specific free-energy. The logarithm is considered element-wise. For the derivations please refer to Appendix \ref{sec:AppendixF}.

\subsubsection{Perception}
According to active inference, both perception and decision making are based on the minimization of free-energy. In particular, for state estimation, we take partial derivatives of $F$ with respect to the states and set the gradient to zero. The {posterior distribution} of the state, conditioned by a plan, is given by:
{
\begin{subequations}
\label{eq:s_taupi}
\begin{gather}
    \bm s_{\tau = 1}^{\pi} = \sigma(\ln{\bm D} + \ln{{(}\bm B^\top_{a_{\tau}} \bm s_{\tau+1}^{\pi}{)}} + \ln{{(}\bm A^\top \bm o_\tau{)})}\\
    \bm s_{1 < \tau < T}^{\pi} = \sigma(\ln{{(}\bm B_{a_{\tau-1}}\bm s_{\tau-1}^{\pi}}{)} + \ln{{(}\bm B^\top_{a_{\tau}} \bm s_{\tau+1}^{\pi}{)}} + \ln{{(}\bm A^\top \bm o_\tau{)})}\\
    \bm s_{\tau = T}^{\pi} = \sigma(\ln{{(}\bm B_{a_{\tau-1}}\bm s_{\tau-1}^{\pi}}{)} + \ln{{(}\bm A^\top \bm o_\tau{)})}
\end{gather}
\end{subequations}}
where $\sigma$ is the softmax function. {The column of $\bm B^\top_{a_{\tau}}$ are normalized.} For the complete derivation, please refer to Appendix~\ref{sec:AppendixS}. {Note that when $\tau = 1$ it holds $\ln{{(}\bm B_{a_{\tau-1}}\bm s_{\tau-1}^{\pi}}{)} = \ln{\bm D}$. We provide below an example of state update and highlight the effect of uncertain action outcomes in the state estimation process.}

{
\begin{example} \textbf{\underline{State estimation}:}
An active inference agent lives in a simple world composed of one state which can have two possible values. The only action the agent can do is to stay still ($\bm B_{idle}$), so $\pi = a_\tau \forall \tau$ with $a_\tau = idle$. However, there are some chances that unwanted transitions between states can occur. The agent can only predict one step ahead ($T=2$) and it receives an observation $\bm o_{\tau=1}$ at the start, that is related to the state through $\bm A$. The agent has no prior information about the initial state or future observations, thus $\bm D$ is uniform as well as the initial guess about the posterior distributions of the state. 
This is modeled as follows:
\begin{equation}
\nonumber
     \bm A = \begin{bmatrix}
        0.9 & 0.1 \\
        0.1 & 0.9
    \end{bmatrix},\ 
    \bm B_{idle} = \begin{bmatrix}
        0.8 & 0.2 \\
        0.2 & 0.8
    \end{bmatrix},\ 
    {\bm D} = \begin{bmatrix} 0.5\\ 0.5
    \end{bmatrix} 
\end{equation}
\begin{equation}
\nonumber
    \bm o_{\tau = 1} = \begin{bmatrix} 1\\ 0
    \end{bmatrix},\ 
     \bm o_{\tau = 2} = \begin{bmatrix} 0\\ 0
    \end{bmatrix},\ 
    \bm s_{\tau = 1}^{\pi} = \bm s_{\tau = 2}^{\pi} = \begin{bmatrix} 0.5\\ 0.5
    \end{bmatrix} 
\end{equation}
The update of the posterior distribution for state estimation is, according to \eqref{eq:s_taupi}:
{
\begin{eqnarray}
\nonumber
    \bm s_{\tau = 1}^{\pi} & = & \sigma\left(\ln\begin{bmatrix} .5\\ .5
    \end{bmatrix} + \ln\left(\begin{bmatrix}
        .8 & .2 \\
        .2 & .8
    \end{bmatrix}\begin{bmatrix} .5\\ .5
    \end{bmatrix}\right)\right.\\ 
    \nonumber
    & + & \left. \ln\left(\begin{bmatrix}
        .9 & .1 \\
        .1 & .9
    \end{bmatrix} \begin{bmatrix} 1\\ 0
    \end{bmatrix}\right)\right) = \begin{bmatrix} .9\\ .1 \end{bmatrix}\\
\nonumber
    \bm s_{\tau = 2}^\pi & = & \sigma\left(\ln \left( \begin{bmatrix}
        .8 & .2 \\
        .2 & .8
    \end{bmatrix}\begin{bmatrix} .9\\ .1
    \end{bmatrix}\right) + \ln \left( \begin{bmatrix}
        .9 & .1 \\
        .1 & .9
    \end{bmatrix} \begin{bmatrix} 0\\ 0
    \end{bmatrix}\right)\right) \\
    \nonumber
    &=& \begin{bmatrix} .74\\ .26 \end{bmatrix}
\end{eqnarray}}
As commonly done in active inference literature, {a small number (for instance $e^{-16}$ \cite{smith2021step}) is added when computing the logarithms. This prevents numerical errors in case of $\ln(0)$.}
Note that, if there would be no uncertainty on the actions the agent can take, i.e. $\bm B_{idle}$ is the identity in this case, then the estimated hidden state at $\tau = 2$ would be $\bm s_{\tau = 2}^{\pi} = [0.9, 0.1]^\top$. The agent would then be more confident under this action.
\end{example}
}

\subsubsection{Expected Free-energy}
Active inference unifies action selection and perception by assuming that actions fulﬁll predictions based on inferred states. Since the internal model can be biased toward preferred states or observations (\textit{prior desires}), active inference induces actions that will bring the current beliefs towards the preferred states. An agent builds beliefs about future states which are then used to compute the expected free-energy. The latter is necessary to evaluate alternative plans. Plans that lead to preferred observations are more likely. Preferred observations are specified in the model parameter $\bm C$. This enables action to realize the next (proximal) \textit{observation} predicted by the plan that leads to (distal) goals. The expected free-energy for a plan $\pi$ at time $\tau$ is given by:
\begin{equation}
        {{G(\pi,\tau) = \underbrace{\bm o_\tau^{\pi\top} \left[\ln{\bm o_\tau^\pi}-\ln \bm C\right]}_{Reward\ seeking} \underbrace{- diag(\bm A ^\top \ln\bm A)^\top\bm s_\tau^\pi}_{Information\ seeking}}}
     \label{eq:expectedF}
\end{equation}
{The $diag()$ function simply takes the diagonal elements of a matrix and puts them in a column vector. This is just a method to extract the correct matrix entries in order to compute the expected free-energy \cite{smith2021step}.}
By minimizing expected free-energy the agent balances reward and information seeking (see Appendix \ref{sec:AppendixG} for derivations).
{
Reward and information-seeking behaviors that arise from the formulation of the expected free-energy are illustrated respectively in Examples \ref{ex:Gexploit} and \ref{ex:Gexplore}. To keep the examples reasonably simple to be computed by hand, we consider that the posterior over states according to different plans have already been computed following \eqref{eq:s_taupi}, and are given.}
{
\begin{example} 
\label{ex:Gexploit} \textbf{\underline{Reward seeking}:} 
Consider an agent has computed the posterior distribution of the states $\bm s_\tau^{\pi_1}$, $\bm s_\tau^{\pi_2}$ under two different plans $\pi_1$ and $\pi_2$. The agent has a preference for a particular value of the state encoded in $\bm C$. The model for this example is the following:
\begin{equation}
\nonumber
     \bm A = \begin{bmatrix}
        0.9 & 0.1 \\
        0.1 & 0.9
    \end{bmatrix},\ 
    {\bm C} = \begin{bmatrix} 1\\ 0
    \end{bmatrix},\ 
    \bm s_\tau^{\pi_1} = \begin{bmatrix} 0.95\\ 0.05
    \end{bmatrix},\ \bm s_\tau^{\pi_2} = \begin{bmatrix} 0.05\\ 0.95
    \end{bmatrix} 
\end{equation}
Let us consider the reward-seeking term in \eqref{eq:expectedF} (note that the information-seeking term is equal in both plans for this example). The observations expected under the two different plans are:
\begin{equation}
\nonumber
    \bm o_\tau^{\pi_1} = \bm A \bm s_\tau^{\pi_1} = \begin{bmatrix} 0.86\\ 0.14
    \end{bmatrix},\ 
    \bm o_\tau^{\pi_2} = \bm A \bm s_\tau^{\pi_2} = \begin{bmatrix} 0.14 \\ 0.86
    \end{bmatrix}
\end{equation}
Intuitively, according to $\bm C$, the first plan is preferable and it should have the lowest expected free-energy because it leads to preferred observations with higher probability. Numerically:
\begin{equation}
\nonumber
    \bm o_\tau^{\pi_1\top} \left[\ln{\bm o_\tau^{\pi_1}}-\ln \bm C\right] = \begin{bmatrix} 0.86\\ 0.14
    \end{bmatrix}^{\top} \left[\ln{\begin{bmatrix} 0.86\\ 0.14
    \end{bmatrix}}-\ln{\begin{bmatrix} 1\\ 0
    \end{bmatrix}}\right] \approx 1.84
\end{equation}
Similarly for $\pi_2$, $\bm o_\tau^{\pi_2\top} \left[\ln{\bm o_\tau^{\pi_2}}-\ln \bm C\right] \approx 13.35$. As can be noticed, the plan that brings the posterior state closest to the preference specified in $\bm C$ leads to the lowest reward-seeking term.
\end{example}}

{
\begin{example}
\label{ex:Gexplore} \textbf{\underline{Information seeking}:} Let us consider a variation of Example~\ref{ex:Gexploit}. The agent is not given any preference for a specific state, and the likelihood matrix $\bm A$ encodes now the fact that observations are expected to provide more precise information when the agent is in the second state (second column of $\bm A$). This results in the following models:
\begin{equation}
\nonumber
     \bm A = \begin{bmatrix}
        0.7 & 0.1 \\
        0.3 & 0.9
    \end{bmatrix},\ 
    {\bm C} = \begin{bmatrix} 0\\ 0
    \end{bmatrix},\ 
    \bm s_\tau^{\pi_1} = \begin{bmatrix} 0.9\\ 0.1
    \end{bmatrix},\ \bm s_\tau^{\pi_2} = \begin{bmatrix} 0.1\\ 0.9
    \end{bmatrix} 
\end{equation}
We expect that the plan which leads to a state with less ambiguous information has the lowest information-seeking term. For the first plan:
\begin{eqnarray}
\nonumber
    &-& diag(\bm A ^\top \ln\bm A)^\top\bm s_\tau^{\pi_1} = 
    \\
    \nonumber 
    &-&diag\left( \begin{bmatrix}
        0.7 & 0.3 \\
        0.1 & 0.9
    \end{bmatrix} \ln{\begin{bmatrix}
        0.7 & 0.1 \\
        0.3 & 0.9
    \end{bmatrix}} \right)^\top\begin{bmatrix} 0.9\\ 0.1
    \end{bmatrix}  \approx 0.58
\end{eqnarray}
For the second plan, the information-seeking term is instead $\approx 0.35$. The state achieved with the second plan generates less ambiguous observations. Plans that lead to the lowest ambiguity in sensory information, and thus minimize $G$, are preferred.
\end{example}
}
{
In more complex examples, minimizing $G$ leads to a balance in reward and information seeking. For a fully-fledged exploration-exploitation problem, we refer the reader to a recently released Python library for active inference \cite{pymdp} which contains an interactive and visual example\footnote{\url{https://pymdp-rtd.readthedocs.io/en/latest/notebooks/cue_chaining_demo.html}} of the emergent behavior of a rat in a grid world which has to collect cues to disclose the location of the reward.}
\paragraph{Planning and decision making}
Taking the gradient of $F$ with respect to plans, and recalling that the generative model specifies the
approximate posterior over plans as a softmax function of the expected free-energy \cite{DaCosta2020} it holds that:
\begin{equation}
    \bm \pi = \sigma(-\bm G_\pi - \bm F_\pi)
\end{equation}
{where the vector $\bm \pi$ {encodes the posterior distribution over plans reflecting the predicted value of each plan.} $\bm F_\pi = (F(\pi_1),F(\pi_2),...)^\top$ and $\bm G_\pi = (G(\pi_1),G(\pi_2),...)^\top$.} See Appendix \ref{sec:AppendixPi} for the details.
\paragraph{Plan independent state-estimation}
Given the probability over $p$ possible plans, and the plan dependent states $\bm s_\tau^\pi$, we can compute the overall probability distribution for the states over time through Bayesian Model Average:
\begin{equation}
    \label{eq:s_tau}
    \bm s_\tau = \sum_i{\bm s_\tau^{\pi_i}\bm\pi_i},\ \textnormal{where}\ i \in \{1,...,p\}
\end{equation}
were $\bm s_\tau^{\pi_i}$ is the probability of a state at time $\tau$ under plan $i$ and $\bm\pi_i$ is the probability of
plan $i$. This is the average prediction for the state at a certain time, so $\bm s_\tau$, according to the probability of each plan. In other words, this is a weighted average over different models. Models with high probability receive more weight, while models with lower probabilities are discounted. 
\paragraph{Action selection}
{The action for the agent to be executed is the first action of the most likely plan:
\begin{eqnarray}
\label{eq:a_t}
    \lambda = \max(\underbrace{[\bm\pi_{1}, \bm\pi_{2},...,\bm\pi_{p}]}_{\bm \pi^\top}),\ 
    a_\tau = \pi_\lambda(\tau = 1)
\end{eqnarray}
where $\lambda$ is the index of the most likely plan.}

\begin{example} 
\label{ex:planSel} \textbf{\underline{Plan and action selection}:}
By computing the expected free-energy of Exercise~\ref{ex:Gexploit} using \eqref{eq:expectedF} and including the information seeking term, we obtain $\bm G_\pi = [G(\pi_1), G(\pi_2)]^\top \approx [2.16, 13.68]^\top$.  For the sake of this example, let us assume the free-energy is equal for both plans, that is $\bm F_\pi = [F(\pi_1), F(\pi_2)]^\top \approx [1.83, 1.83]^\top$. The posterior distribution over plans is then:
\begin{equation}
    \nonumber
    \bm \pi = \sigma \left(-\begin{bmatrix} 2.16\\ 13.68
    \end{bmatrix} -\begin{bmatrix} 1.83\\ 1.83
    \end{bmatrix} \right) \approx \begin{bmatrix} 0.99\\ 0.01
    \end{bmatrix}
\end{equation}
As can be seen, the most likely plan is the first one, in accordance with the conclusions of Exercise~\ref{ex:Gexploit}. The action to be applied by the agent is the first action of $\pi_1$. 
\end{example}

The active inference algorithm is summarised in pseudo-code in Algorithm~\ref{alg:AIC}.
\begin{algorithm}
\caption{Action selection with active inference}\label{alg:AIC}
\begin{algorithmic}[1]
    \State \texttt{Set} $\bm C$ \Comment{prior preferences} 
\For{$\tau=1:T$}
    \State {If not specified, get state from $\bm D$ if $\tau == 1$} 
    \State If not specified, get observation from $\bm A$ 
    \State Compute $F$ for each plan\Comment{eq. \eqref{eq:free-energy}} 
    \State Update posterior state $\bm s_\tau^\pi$ \Comment{eq. \eqref{eq:s_taupi}}
    \State Compute $G$ for each plan \Comment{eq. \eqref{eq:expectedF}}
    \State Bayesian model averaging \Comment{eq. \eqref{eq:s_tau}}
    \State Action selection \Comment{eq. \eqref{eq:a_t}}
\EndFor
\State \textbf{Return} $a$ \Comment{Preferred action}
\end{algorithmic}
\end{algorithm}

\subsubsection{Multiple sets of states and observations}
{The active inference model introduced in this section can also handle multiple sets of independent states and observations \cite{smith2021step}. A famous example in the active inference literature covers a rat in a T-maze \cite{friston2017active}. The rat is seeking a reward (cheese) but the location of the cheese is only known after receiving a cue. In this case, one set of states encodes the location of the rat, while another set encodes the initially unknown location of the cheese. Each independent set of states is called a \textit{state factor}. In the same way, there can be multiple sets of observations coming from different sensors. Each set is a different \textit{observation modality}. Works as \cite{smith2021step,hesp2021} provide a step-by-step tutorial on active inference with fully worked out toy examples including multiple factors and modalities. We provide explicit models for our robotic case with multiple state factors and observations in Sec.~\ref{sec:modela_idea}. }

\subsection{Background on BTs}
We now describe the high-level concepts at the basis of BTs according to previous work such as \cite{Colledanchise2018, Colledanchise2017}. These concepts will be useful to understand the novel hybrid scheme proposed in the next section. A BT is a directed tree composed of nodes and edges that can be seen as a graphical modeling language. It provides a structured representation for the execution of actions that are based on conditions and observations in a system. The nodes in a BT follow the classical definition of parents and children. The root node is the only node without a parent, while the leaf nodes are all the nodes without children. In a BT, the nodes can be divided into control flow nodes (\textit{Fallback}, \textit{Sequence}, \textit{Parallel}, or \textit{Decorator}), and into execution nodes (\textit{Action} or \textit{Condition}) which are the leaf nodes of the tree. When executing a given BT in a control loop, the root node sends a \textit{tick} to its child.  A tick is nothing more than a signal that allows the execution of a child. The tick propagates in the tree following the rules dictated by each control node. A node returns a status to the parent, which can be \textit{running} if its execution has not finished yet, \textit{success} if the goal is achieved, or \textit{failure} in the other cases. At this point, the return status is propagated back up the tree, which is traversed again following the same rules. The most important control nodes are:
\paragraph*{Fallback nodes} A fallback node ticks its children from left to right. It returns \textit{\textit{success}} (or \textit{running}) as soon as one of the children returns \textit{success} (or \textit{running}). When a child returns \textit{success} or \textit{running}, the fallback does not tick the next child, if present. If all the children return \textit{failure}, the fallback returns \textit{failure}. This node is graphically identified by a gray box with a question mark "$?$";
\paragraph*{Sequence nodes} The sequence node ticks its children from left to right. It returns \textit{running} (or \textit{failure}) as soon as a child returns \textit{running} (or \textit{failure}). The sequence returns \textit{success} only if all the children return \textit{success}. If a child returns \textit{running} or \textit{failure}, the sequence does not tick the next child, if present. In the library, we used to implement our BTs \cite{bt_library} the \textit{sequence} node, indicated with [$\rightarrow$], keeps ticking a running child, and it restarts only if a child fails. \cite{bt_library} provides also \textit{reactive sequences} [$\rightarrow^R$] where every time a sequence is ticked, the entire sequence is restarted from the first child.

The execution nodes are \textit{Actions} and \textit{Conditions}:
\paragraph*{Action nodes} An Action node performs an action in the environment. While an action is being executed, this node returns \textit{running}. If the action is completed correctly it returns \textit{success}, while if the action cannot be completed it returns \textit{failure}. Actions are represented as  red rectangles;
\paragraph*{Condition nodes} A Condition node determines if a condition is met or not, returning \textit{success} or \textit{failure} accordingly. Conditions never return \textit{running} and do not change any states or variables. They are represented as orange ovals;

An example BT is given in Fig.~\ref{fig:btBackground}. 

\begin{figure}[!htb]
    \centering
    \includegraphics[width=0.4\textwidth]{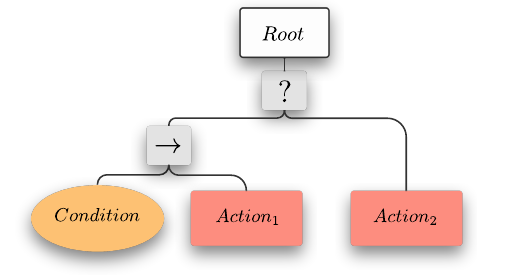}
    \caption{Example of BT. The \textit{tick} traverses the tree starting from the \textit{Root}. If $Condition$ is true $Action_1$ is executed. Then, if $Action_1$ returns success, the \textit{Root} returns success, otherwise $Action_2$ is executed.}
    \label{fig:btBackground}
\end{figure}

%% file: tex/myIdea.tex
\section{Active Inference and BTs for Reactive Action Planning and Execution}
\label{sec:Solution}
In this section, we introduce our novel approach using BTs and active inference. Even though active inference is a very promising theory, from a computational perspective computing the expected free-energy for each possible plan that a robot might take is cost-prohibitive. This curse of dimensionality is due to the combinatorial explosion when looking deep into the future \cite{DaCosta2020}. To solve this problem, we propose to replace \textit{deep plans} with \textit{shallow decision trees} that are \textit{hierarchically composable}. This will allow us to simplify our offline plans, exploit opportunities, and act intelligently to resolve local unforeseen contingencies. Our idea consists of two main intuitions:
\begin{itemize}
    \item To avoid combinatorial explosion while planning and acting with active inference for long-term tasks we specify the nominal behavior of an agent through a BT, used as a prior. In doing so, BTs provide \textit{global reactivity to foreseen situations}
    \item To avoid coding every possible contingency in the BT, we program only desired states offline, and we leave action selection to the online active inference scheme. Active inference provides then \textit{local reactivity to unforeseen situations}.
\end{itemize}

To achieve such a hybrid integration, and to be able to deploy this architecture on real robotic platforms, we addressed the following three fundamental problems: 1) how to define the generative models for active inference in robotics, 2) how to use BTs to provide priors as desired states to active inference, 3) how to handle action preconditions in active inference, and possible conflicts which might arise at run-time in a dynamic environment.

\subsection{Definition of the models for active inference}
\label{sec:modela_idea}
The world in which a robot operates needs to be abstracted such that the active inference agent can perform its reasoning processes. In this work, we operate in a continuous environment with the ability of sensing and acting through symbolic decision making. {In the general case, the decision making problem will include multiple sets of states, observations, and actions. Each independent set of states is a factor, for a total of $n_f$ factors. For a generic factor $f_j$ where $j\in\mathcal{J} = \{1,...,n_f\}$, the corresponding state factor is: 
\begin{gather}
    \nonumber
    s^{(f_j)} = \left[s^{(f_j,1)}, s^{(f_j,2)},...,s^{(f_j,m^{(f_j)})}\right]^\top,\\ 
    \mathcal{S} = \big\{ s^{(f_j)}|j\in\mathcal{J}\big\}
\end{gather}
where $m^{(f_j)}$ is the number of mutually exclusive symbolic values that a state factor can have. Each entry of $s^{(f_j)}$ is a real value between 0 and 1, and the sum of the entries is 1. {This represents the current belief state.}} 
Then, we define $x\in \mathcal{X}$ the continuous states of the world and the internal states of the robot are accessible through the symbolic perception system. The role of this perception system is to 
{compute the symbolic observations based on the continuous state $x$}, such that they can be manipulated by the discrete active inference agent. Observations $o$ are used to build a probabilistic belief about the current state. {Assuming one set of observations per state factor with $r^{(f_j)}$ possible values, it holds:
\begin{gather}
    \nonumber
    o^{(f_j)} = \left[o^{(f_j,1)}, o^{(f_j,2)},...,o^{(f_j,r^{(f_j)})}\right]^\top,\\ 
    \mathcal{O} = \big\{ o^{(f_j)}|j\in\mathcal{J}\big\}
\end{gather}
Additionally, the robot has a set of symbolic skills to modify the corresponding state factor:
\begin{gather}
    \nonumber
    a_\tau \in \alpha^{(f_j)} = \big\{a^{(f_j,1)}, a^{(f_j,2)},...,a^{(f_j,k^{(f_j)})}\big\},\\ 
    \mathcal{A} = \big\{ \alpha^{(f_j)}|j\in\mathcal{J}\big\}
\end{gather}
where $k^{(f_j)}$ is the number of actions that can affect a specific state factor $f_j$. Each generic action $a^{(f_j,\cdot)}$ has associated a symbolic name, \textit{parameters}, \textit{pre-} and \textit{postconditions}}:
\begin{table}[ht]
\centering 
\begin{tabular}{l c c} 
    \textbf{Action} $a^{(f_j,\cdot)}$& \textbf{Preconditions}& \textbf{Postconditions}\\  %
    \texttt{action\_name(}$par$\texttt{)}&\texttt{prec}$_{a^{(f_j,\cdot)}}$ & \texttt{post}$_{a^{(f_j,\cdot)}}$
\end{tabular}
\end{table}

where \texttt{prec}$_{a^{(f_j,\cdot)}}$ and \texttt{post}$_{a^{(f_j,\cdot)}}$ are \textit{first-order logic predicates} that can be evaluated at run-time. A logical predicate is a boolean-valued function $\mathcal{P}:\mathcal{X}\rightarrow\{$\texttt{true}, \texttt{false}$\}$. 

Finally, {we define the logical state $l^{(f_j)}$ as a one-hot encoding of $s^{(f_j)}$. We indicate as $\mathcal{L}_c(\tau) = \big\{ l^{(f_j)}|j\in\mathcal{J}\big\}$ the (time varying) \textit{current} logical state of the world}. Defining a logic state based on the probabilistic belief $s$ built with active inference, instead of directly using the observation of the states $o$, increases robustness against noisy sensor readings, as we will explain in Example~\ref{ex:noise}. 
Given the model of the world just introduced, we can now define {for each factor $f_j$} the likelihood matrix $\bm A^{(f_j)}$, {the $k^{(f_j)}$ transition matrices $\bm B_{a^{(f_j,\cdot)}_\tau}$ and the prior preferences $\bm C^{(f_j)}$}. When an observation is available, $\bm A^{(f_j)}$ provides information about the corresponding value of a state factor $s^{(f_j)}$. For a particular state, the probability of a state observation pair { $o^{(f_j)}_\tau$, $s^{(f_j)}_\tau$ is given by $\bm A^{(f_j)} \in \mathbb{R}^{r^{(f_j)} \times m^{(f_j)}}$. In case a state factor is observable with full certainty, each state maps into the corresponding observation thus the likelihood matrix is the identity of size ${m^{(f_j)}}$, $I_{m^{(f_j)}}$}.
Note that knowing the mapping between observations and states does not necessarily mean that we can observe all the states at all times. Observations can be present or not, and when they are the likelihood matrix indicates the relation between that observation and the state. This relation can be more complex and incorporate uncertainty in the mapping as well. 
To define the transition matrices, we need to encode in a matrix form the effects of each action on the relevant state factors. The probability of a ending up in a state {$s^{(f_j)}_{\tau+1}$, given $s^{(f_j)}_{\tau}$ and action $a^{(f_j,\cdot)}_\tau$ is given by:
\begin{gather}
       \nonumber
       P(s^{(f_j)}_{\tau+1}|s^{(f_j)}_{\tau},a^{(f_j,\cdot)}_\tau) = Cat(\bm B_{a^{(f_j,\cdot)}_\tau}s^{(f_j)}_{\tau}),\\ 
       \bm B_{a^{(f_j,\cdot)}_\tau} \in \mathbb{R}^{m^{(f_j)} \times m^{(f_j)}}
\end{gather}
In other words, we define $\bm B_{a^{(f_j,\cdot)}_\tau}$ as a square matrix encoding the post-conditions of action $a^{(f_j,\cdot)}_\tau$. The prior preferences over observations (or states) need to be encoded, {for each factor,} in $\bm C^{(f_j)} \in \mathbb{R}^{m^{(f_j)}}$, with $\mathcal{C} = \big\{\bm C^{(f_j)}|j\in\mathcal{J}\big\}$.} The higher the value of a preference, the more preferred a particular state is, and vice-versa. Priors are formed according to specific desires and they will be used to interface active inference and BTs. Finally, one has also to define the vector encoding the initial belief about the probability distribution of the states, that is {$\bm D^{(f_j)} \in \mathbb{R}^{m^{(f_j)}}$. This vector is normalized, and when no prior information is available, each entry will be $1/m^{(f_j)}$.} In this work, we assume the model parameters such as likelihood and transition matrices to be known. However, one could use the free-energy minimization to learn them \cite{friston2016AILearning}.

\begin{example}
\label{ex:obs_states}
\textit{Consider a mobile manipulator in a retail environment. {We want the robot to be able to decide when to navigate to a certain goal location. To achieve so we need one state factor ${s^{(loc)}}$, one observation ${o^{(loc)}}$, and one symbolic action $a^{(loc,1)} = $ \texttt{moveTo(goal)}. The robot can also decide not to do anything, so  $a^{(loc,2)} = $ \texttt{idle}}.}
\begin{table}[h]
\centering 
\begin{tabular}{l l l} 
    \textbf{Actions}& \textbf{Preconditions}& \textbf{Postconditions}\\  %
    \texttt{moveTo(goal)} & \texttt{-} & ${l^{(loc)}} = [1\ 0]^\top$ \\
    \texttt{idle} & \texttt{-} & \texttt{-} \\ 
\end{tabular}
\end{table}

\textit{The current position in space of the robot {is a continuous value $ x\in  \mathcal{X}$}. However, during execution the robot is given an observation ${o^{(loc)}}$ which indicates simply if the} \texttt{goal} \textit{has been reached or not. In this case $\bm A^{(loc)} = I_2$. The agent is then constantly building a probabilistic belief ${s^{(loc)}}$ encoding the chances of being at the} \texttt{goal}. \textit{In case the robot has not yet reached the goal, a possible configuration at time $\tau$ is the following:}
\begin{equation}
    {o^{(loc)}} = \begin{bmatrix} \texttt{isAt(goal)}\\ \texttt{!isAt(goal)}
    \end{bmatrix} = \begin{bmatrix} 0 \\ 1
    \end{bmatrix}{,}\ {s^{(loc)}} = \begin{bmatrix} 0.08 \\ 0.92
    \end{bmatrix}{,}
\end{equation}
\begin{equation}
\nonumber
   \ {l^{(loc)}} = \begin{bmatrix} 0 \\ 1
    \end{bmatrix}{,}\ {\bm B_{moveTo}} = \begin{bmatrix}
        0.95 & 0.9 \\
        0.05 & 0.1
    \end{bmatrix}{,}\ \bm B_{idle} = \begin{bmatrix}
        1 & 0 \\
        0 & 1
    \end{bmatrix}
\end{equation}
{The preference over a state to be reached is given through $\bm C^{(loc)}$. A robot wanting to reach a location will have, for instance, a preference $\bm C^{(loc)} = [1, 0]^\top$.}
\vspace{0mm}
\end{example}
The transition matrix {$\bm B_{moveTo}$} encodes the probability of reaching a goal location through the action \texttt{moveTo(goal)}, which might fail with a certain probability. We also encode an \texttt{Idle} action, which does not modify any state, but it provides information on the outcome of the action selection process as we will see in the next subsections. In this simple case, the world state is just a single state factor ${s^{(loc)}}$.  On the other hand, in later more complicated examples, the world state will contain all the different aspects of the world, for which a probabilistic representation of their value is built and updated.

Using the proposed problem formulation for the active inference models, we can abstract redundant information which is not necessary to make high-level decisions. For instance, in the example above, we are not interested in building a probabilistic belief of the current robot position. To decide if to use the action \texttt{moveTo(goal)} or not, it is sufficient to encode if the \texttt{goal} has been reached or not. 

\subsection{BTs integration: planning preferences, not actions}
\label{sec:subsecPlanOffline}
To achieve a meaningful behavior in the environment through active inference, we need to encode specific desires into the agent's brain through {$\mathcal{C} = \big\{\bm C^{(f_j)}|j\in\mathcal{J}\big\}$.}

\paragraph*{A prior as BT}
We propose to extend the available BT nodes in order to be able to specify desired states to be achieved as leaf nodes. We introduce a new type of leaf node called \textit{prior nodes}, indicated with a green hexagon. These nodes can be seen as standard action nodes but instead of commanding an action, they simply set the desired value of a state in {$\mathcal{C}$} and they run active inference for action selection. The prior node is then just a leaf node in the BT which returns: \texttt{Success} if a state is achieved, \texttt{Running} while trying to achieve it, or \texttt{Failure} if for some reason it is not possible to reach the desired state. The return statuses are according to the outcome of our reactive action selection process as explained in Section \ref{sec:completeControlSheme}.

\paragraph*{Sub-goals through BTs}
To reach a distal goal state, we plan achievable sub-goals in the form of desired logical states $l$, according to the available actions that a robot possesses. This idea of using sub-goals was already used in \cite{Kaplan2018}, but in our solution with BTs, we provide a task-independent way to define sub-goals which is only based on the set of available skills of the robot, such that we can make sure that it can complete the task. At planning time, we define the ideal sequence of states and actions to complete a task such that subsequent sub-goals (or logical desired states) are achievable by means of one action.  This can be done manually or through automated planning. At run-time, however, we only provide the sequence of states to the algorithm, as in Fig.~\ref{fig:planning}.
\begin{figure}[!htb]
    \centering
    \includegraphics[width=0.42\textwidth]{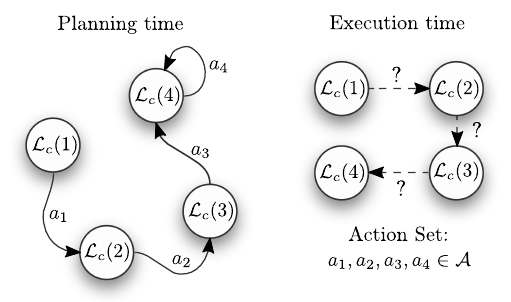}
    \caption{The path among states is planned offline using the available set of actions but only the sequence of states is provided at run-time. Actions are chosen online from the available set with active inference.}
    \label{fig:planning}
\end{figure}

\begin{example}
\label{ex:simpleBT}
\textit{To program the behavior of the robot in Example~\ref{ex:obs_states} to visit a certain goal location, the BT will set the prior over ${s^{(loc)}}$ to ${\bm C^{(loc)}} = [1, 0]^\top$ meaning that the robot would like to sense to be at} \texttt{goal}.
\end{example}
A classical BT and a BT for active inference with prior nodes are reported in Fig.~\ref{fig:btSimple}. Note that the action is left out in the BT for active inference because these are selected at runtime. In this particular case, the condition \texttt{isAt(goal)} can be seen as the desired observation to obtain. 
\begin{figure}[!htb]
    \centering
    \includegraphics[width=0.4\textwidth]{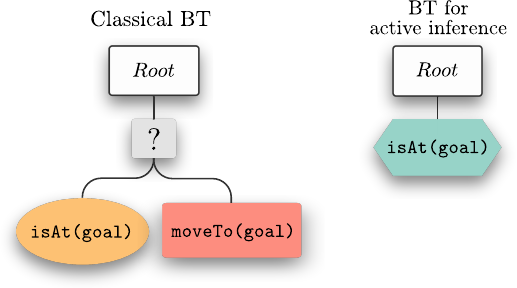}
    \caption{BT to navigate to a location using a classical BT and a BT for active inference. One action \texttt{moveTo(goal)} is available and one condition \texttt{isAt(goal)} provides information if the current location is at the goal. The prior node for active inference (green hexagon) sets the desired prior and runs the action selection process.} 
    \label{fig:btSimple}
\end{figure}

Note that the amount of knowledge (i.e. number of states and actions) that is necessary to code a classical BT or our active inference version in Example \ref{ex:simpleBT} is the same. However, we abstract the fallback by planning in the state space and not in the action space. Instead of programming the action \texttt{moveTo(goal)} we only set a prior preference over the state \texttt{isAt(goal)} since the important information is retained in the state to be achieved rather than in the sequence of actions to do so. Action selection through active inference will then select the appropriate skills to match the current state of the world with the desired one, minimizing this discrepancy through free-energy minimization.
\begin{example}
\textit{Consider the scenario in Example~\ref{ex:obs_states} with the prior as in Example~\ref{ex:simpleBT}. The prior is specifying a preference over being at the desired goal location, but the mobile manipulator is not. The plan generated at runtime {with Alg.~\ref{alg:AIC}} (see Examples \ref{ex:Gexploit} and \ref{ex:planSel}) would be to perform the action} \texttt{moveTo(goal)}, \textit{since this increases the probability of getting an observation} \texttt{isAt(goal)}.
\end{example}
As we will thoroughly explain in Sec.~\ref{sec:Analysis_DoA}, our algorithm creates online a stable region of attraction  from the current state to the goal state instead of planning it offline through fallbacks. See Example~\ref{ex:doa} for a concrete case.
\begin{figure*}
    \centering
    \includegraphics[width=0.85\textwidth]{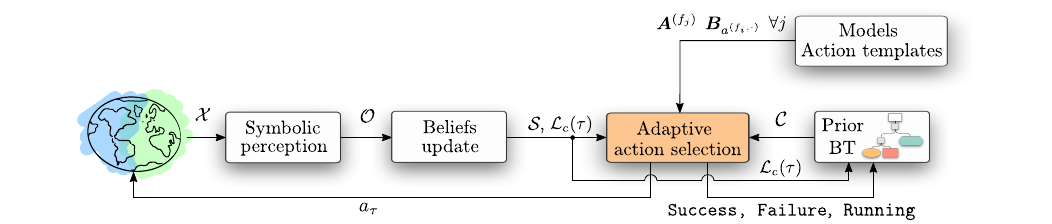}
    \caption{Overview of the control architecture for reactive action planning and execution using active inference. Adaptive action selection is performed according to Algorithm \ref{alg:idea}. {The symbolic perception module {computes the symbolic observations based on the continuous state}. In this work we assume this mapping known and encode it through simple rules based on measurements from the robot's sensors. However learning methods to define this relationship from data could be employed as well.}}
    \label{fig:scheme}
\end{figure*}

\subsection{Action preconditions and conflicts}
Past work on active inference, such as \cite{Kaplan2018}, was based on the assumption that actions were always executable and non-conflicting, but these do not hold in more realistic scenarios. 

\paragraph*{Action preconditions in active inference} We propose to encode action preconditions as desired logical states that need to hold to be able to execute a particular action. This is illustrated in the next example. 

\begin{example}
\label{ex:movePick}
\textit{We add one more action to the set of skills of our mobile manipulator:} \texttt{pick(obj)} \textit{and the relative transition matrix {$\bm B_{pick}$}. The action templates are extended as follows:}
\begin{table}[h]
\centering 
\begin{tabular}{l l l} 
    \textbf{Actions}& \textbf{Preconditions}& \textbf{Postconditions}\\  %
    \texttt{moveTo(goal)} & \texttt{-} & ${l^{(loc)}} = [1\ 0]^\top$ \\
    & & ${l^{(reach)}} = [1\ 0]^\top$ \\
    \texttt{pick(obj)} & \texttt{isReachable(obj)}& ${l^{(hold)}} = [1\ 0]^\top$\\
\end{tabular}
\end{table}
\begin{align}
\nonumber
     & {o^{(hold)}} = \begin{bmatrix} \texttt{isHolding(obj)}\\ \texttt{!isHolding(obj)} \end{bmatrix}{,}\hspace{5mm}
    {B_{pick}} = \begin{bmatrix} 0.95 & 0.9 \\ 0.05 & 0.1
    \end{bmatrix} & \\
    & {o^{(reach)}} = \begin{bmatrix} \texttt{isReachable(obj)}\\ \texttt{!isReachable(obj)} \end{bmatrix} & 
\end{align}
\vspace{0mm}
\end{example}
where we added a new logical state ${l^{(hold)}}$, the relative belief ${s^{(hold)}}$ and observation ${o^{(hold)}}$, which indicates if the robot is holding the object \texttt{obj}. In the simplest case, we suppose that the only precondition for successful grasping is that \texttt{obj} is reachable. We then add a logical state ${l^{(reach)}}$, as well as ${s^{(reach)}}$ and ${o^{(reach)}}$, to provide active inference with information about this precondition. ${o^{(reach)}}$ can be built for instance trying to compute a grasping pose for a given object. The robot can act on the state ${l^{(hold)}}$ through \texttt{pick}, and it can act on ${l^{(reach)}}$ through \texttt{moveTo}.

\paragraph*{Conflicts resolution in active inference}
Conflict resolution due to dynamic changes in the environment and online decision making is handled through modification of the prior preferences in ${\mathcal{C}}$. The BT designed offline specifies at runtime the desired state to be achieved. This is done by populating the prior preference over a state with the value of one. Note that if there is no goal, the preferences over states are set to zero everywhere so there is no incentive to act to achieve a different state. Given some preferences over states, the online decision making algorithm with active inference selects an action, and checks if the preconditions are holding according to the current belief state. If so, the action is executed, if not, the missing preconditions are added to the current preferred state with a higher preference (i.e. $>1$), in our case with value 2. This can lead to conflicts with the original BT \cite{Colledanchise2019}, that is the robot might want to simultaneously achieve two conflicting states. However, the state relative to a missing precondition has higher priority (i.e. $>1$) by construction. As explained in Algorithm~\ref{alg:idea}, if preconditions are missing at runtime, action selection is performed again with the updated prior ${\mathcal{C}}$, such that actions that will satisfy them are more likely. With our method, there is no need to explicitly detect a conflict and re-ordering a BT then, as was done in past work on the dynamic expansion of BTs \cite{Colledanchise2019}. In fact, since the decision making with active inference happens continuously during task execution, once a missing precondition is met this is removed from the current desired state. Thus, the only remaining preference is the one imposed by the BT, which can now be resumed. This leads to a natural conflict resolution and plan resuming without ad-hoc recovery mechanisms. The advantage of active inference is that we can represent \textit{which} state is important but also \textit{when} with different values of preference. Since missing preconditions are added to the current prior with a higher preference with respect to the offline plan, this will induce a behavior that can initially go against the initial BT because the new desire is more appealing to be satisfied. Conflict resolution is then achieved by \textit{locally updating} prior desires about a state, giving them higher preference. The convergence analysis of this approach is reported in Section~\ref{sec:Analysis}, and a concrete example of conflict resolution in a robotic scenario is presented in Sec.~\ref{sec:runtimeconflicts}, Example~\ref{ex:conflict}.
\subsection{Complete control scheme}
\label{sec:completeControlSheme}
Our solution is summarised in Algorithm~\ref{alg:idea} and Fig.~\ref{fig:scheme}. {Every time a BT is ticked, given a certain frequency, Algorithm~\ref{alg:idea} is run}. The symbolic perception layer takes the sensory readings and translates these continuous quantities into logical observations. This can be achieved through user-defined models according to the specific environment and sensors available. The logical observations are used to perform belief updating to keep a probabilistic representation of the world in $\mathcal{S}$. Then, the logical state $\mathcal{L}_c(\tau)$ is formed. Every time a \textit{prior} node in the BT is ticked, the corresponding priors in ${\mathcal{C}}$ are set. 

For both missing preconditions and conflicts, high priority priors are removed from the preferences ${\mathcal{C}}$ whenever the preconditions are satisfied or the conflicts resolved (lines \texttt{6-10}), allowing to resume the nominal flow of the BT. Active inference from Algorithm~\ref{alg:AIC} is then run for action selection. If no action is required since the logical state corresponds to the prior, the algorithm returns \texttt{Success}, otherwise, the selected action's preconditions are checked and eventually pushed with higher priority. Then, action selection is performed with the updated prior. This procedure is repeated until either an executable action is found, returning \texttt{Running}, or no action can be executed, returning \texttt{Failure}. The case of \texttt{Failure} is handled through the global reactivity provided by the BT. This creates dynamic and stable regions of attraction as explained in Section~\ref{sec:Analysis_DoA}, by means of sequential controller composition \cite{Burridge1999} (lines \texttt{17-31} of Algorithm~\ref{alg:idea}). 

Crucially, in this work, we propose the new idea of using \textit{dynamic priors}. For a factor {$f_j$, $\bm C^{(f_j)}$} is not fixed a priori as in past active inference works, but instead, {it can change over time according to the BT for a task. This allows preconditions checking and conflict resolution within active inference.} A robot can follow a long programmed routine while autonomously taking decisions to locally compensate for unexpected events.  This reduces considerably the need for hard-coded fallbacks, allowing to compress the BT to a minimal number of nodes. 

\begin{algorithm}
\caption{Pseudo-code for Adaptive Action Selection} \label{alg:idea}
\begin{algorithmic}[1]
\State \textit{Get desired prior and parameters from BT:}
    \State ${\mathcal{C}}$, \texttt{param} $\leftarrow$ BT {\Comment{With priority 1}}
\State \textit{Set current observations, beliefs and logical state:}
    \State \texttt{Set} $\mathcal{O},\ \mathcal{S},\ \mathcal{L}_c(\tau)$
    \State \textit{Remove preferences with high-priority {(i.e. $> 1$)} if satisfied:}
\For{{all priors $\bm C^{(f_j)}$ with preference $\geq 1$}} {
    \If{{$l_c^{(f_j)}$ holds}}
        \State {\texttt{Remove pushed preference for} $l_c^{(f_j)}$};
    \EndIf}
\EndFor
\State \textit{Run active inference given} $\mathcal{O},\ \mathcal{S}$ \textit{and} ${\mathcal{C}}$:
    \State $a_\tau$ $\leftarrow$ \texttt{Action\_selection($\mathcal{O},\mathcal{S},{\mathcal{C}}$)}\Comment{Alg. \ref{alg:AIC}}
    \State {\texttt{Update} $\mathcal{L}_c(\tau)$}
    \If{$a_\tau$ == \texttt{Idle}}
        \State \textbf{return} \texttt{Success}; \Comment{No action required}
   \Else
   \State \textit{Check action preconditions:} 
    \While{$a_\tau$ \texttt{!=Idle}}
            \If{\texttt{prec}$_{a_\tau} \in \mathcal{L}_c(\tau)$ {OR \texttt{prec}$_{a_\tau} = \emptyset$}}
               \State \texttt{Execute($a_\tau$)};
               \State \textbf{break, return} \texttt{Running}; \Comment{Executing $a_\tau$}
            \Else
            \State \textit{Push missing preconditions in} ${\mathcal{C}}$:
                \State ${\mathcal{C}} \leftarrow$ \texttt{prec}$_{a_\tau}$; \Comment{With priority 2}
                \State \textit{Exclude $a_\tau$ and re-run Alg. \ref{alg:AIC}}:
                \State \texttt{Remove($a_\tau$)};
                \State $a_\tau$ $\leftarrow$ \texttt{Action\_selection($\mathcal{O},\mathcal{S},{\mathcal{C}}$)}
                \If{\texttt{$a_\tau==$ Idle }}
                    \State \textbf{return} \texttt{Failure}; \Comment{No solution}
                \EndIf
            \EndIf
     \EndWhile
   \EndIf
\end{algorithmic}
\end{algorithm}

%% file: tex/AImodularizeBT.tex
\section{Theoretical analysis}
\label{sec:Analysis}
\subsection{Analysis of convergence}
\label{sec:Analysis_DoA}
We provide a theoretical analysis of the proposed control architecture. There are two possible scenarios that might occur at run-time. Specifically, the environment might or might not differ from what has been planned offline through BTs. These two cases are analyzed in the following to study the convergence to the desired goal of our proposed solution.

\subsubsection{The dynamic environment IS as planned}
In a nominal execution, where the environment in which a robot is operating is the same as the one at planning time, there is a one-to-one equivalence between our approach and a classical BT formulation. This follows directly by the fact that the BT is defined according to Section \ref{sec:subsecPlanOffline}, so each subsequent state is achievable by means of one single action. At any point of the task, the robot finds itself in the planned state and has only one preference over the next state given by the BT through $\mathcal{C}$. The only action which can minimize the expected free-energy is the one used during offline planning. In a nominal case, then, we maintain all the properties of BTs, which are well explained in \cite{Colledanchise2017}. In particular, the behavior will be Finite-Time Successful (FTS) \cite{Colledanchise2017} if the atomic actions are assumed to return success after a finite time. Note that so far we did not consider actions with the same postconditions. However, in this case, Algorithm~\ref{alg:idea} would sequentially try all the alternatives following the given order at design time. This can be improved for instance by making use of semantic knowledge at runtime to inform the action selection process about preferences over actions to achieve the same outcome. This information can be stored for instance in a knowledge base and can be used to parametrize the generative model for active inference.

\subsubsection{The dynamic environment IS NOT as planned}
The most interesting case is when a subsequent desired state is not reachable as initially planned. As explained before, in such a case we push the missing preconditions of the selected action into the current prior $\mathcal{C}$ to locally and temporarily modify the goal. We analyze this idea in terms of sequential controllers (or actions) composition as in \cite{Burridge1999}, and we show how Algorithm~\ref{alg:idea} generates a plan that will eventually converge to the initial goal. First of all, we provide some assumptions and definitions that will be useful for the analysis.
\paragraph*{Assumption 1} The action templates with pre- and postconditions provided to the agent are correct;
\paragraph*{Assumption 2} A given desired goal is achievable by at least one atomic action;
\paragraph*{Definition 1} The domain of attraction of an action $a_i$ is defined as the set of its preconditions. This domain for $a_i$ is indicated as $\mathcal{D}(a_i)$;
\paragraph*{Definition 2} We say that an action $a_1$ \textit{prepares} action $a_2$ if the postconditions $\mathcal{P}_c$ of $a_2$ lie within the domain of attraction of $a_1$, so $\mathcal{P}_c(a_2)\subseteq \mathcal{D}(a_1)$;
\paragraph*{Note}{For the derivations in this section we consider, without lack of generality, one single factor such that we can drop the superscripts, i.e. $\bm C^{(f_j)} = \bm C$}.

Following Algorithm~\ref{alg:idea} each time a prior leaf node is ticked in the BT, active inference is used to define a sequence of actions to bring the current state towards the goal. It is sufficient to show, then, that the asymptotically stable equilibrium of a generic generated sequence is the initial given goal. 

\begin{lemma}
\label{lemma:2}
Let $\mathcal{L}_c(\tau)$ be the current logic state of the world, and $\mathcal{A}$ the set of available actions. An action $a_i\in\mathcal{A}$ can only be executed within its domain of attraction, so when {$\mathcal{L}_c(\tau) \in \mathcal{D}(a_i)$}. Let us assume that the goal encoded in $\bm C$ is a postcondition of an action $a_1$ such that $\mathcal{P}_c(a_1)= \bm C$, and that $\mathcal{L}_c(\tau)\ne \bm C$. If {$\mathcal{L}_c(\tau) \not\in \mathcal{D}(a_1)$}, Algorithm~\ref{alg:idea} generates a plan $\pi=\{a_1,\dots,a_N\}$ with domain of attraction $\mathcal{D}(\pi)$ according to the steps below.
\begin{enumerate}
    \item Let the initial sequence contain $a_1\in\mathcal{A}$, $\pi(1)=\{a_1\}$, $\mathcal{D}(\pi)=\mathcal{D}(a_1)$, set $N = 1$
    \item Remove $a_N$ from the available actions, and add the unmet preconditions {$\mathcal{D}(a_N)$ to the prior $\bm C$ with higher priority, such that} $\bm C =
    \bm C\cup\mathcal{D}(a_N)$
    \item Select $a_{N+1}$ through active inference (Algorithm~\ref{alg:AIC}). Then,  $a_{N+1}$ prepares  $a_{N}$ by construction, $\pi(N+1)=\pi(N)\cup \{a_N\}$, $\mathcal{D}_{N+1}(\pi)=\mathcal{D}_{N}(\pi)\cup\mathcal{D}(a_{N+1})$, and $N = N+1$
    \item Repeat 2, 3 until {$\mathcal{L}_c(\tau) \in \mathcal{D}(a_N)$} \texttt{OR $a_N==$ Idle}  
\end{enumerate}
{If $\mathcal{L}_c(\tau) \in \mathcal{D}(a_N)$, the sequential composition $\pi$ with region of attraction $\mathcal{D}(\pi) = \bigcup_{a_i} \mathcal{D}(a_i)$ for $i = 1, ..., N$} stabilizes the system at the given desired state $\bm C$. If $a_i\in\mathcal{A}$ are FTS, then $\pi$ is FTS. 
\end{lemma} 

\begin{proof}
Since {$\mathcal{L}_c(\tau)\in\mathcal{D}(a_N)$} and $\mathcal{P}_c(a_N)\subseteq\mathcal{D}(a_{N-1})$, it follows that $\mathcal{L}_c(\tau)$ is moving towards $\bm C$. {Moreover, by construction $\mathcal{D}(a_1) \subseteq \bigcup_{a_i}\mathcal{P}_c(a_i)$ for $i = 2, ..., N$. After completing action $a_1$, it results} $\mathcal{L}_c(\tau)\equiv \bm C$ since by definition $\mathcal{P}_c(a_1)= \bm C$.
\end{proof}

Note that if {$\mathcal{L}_c(\tau)\in\mathcal{D}(a_N)$} does not hold after sampling all available actions, it means that the algorithm is unable to find a set of actions that can satisfy the preconditions of the initially planned action. This situation is a major failure that needs to be handled by the overall BT. 
\textit{Lemma} \ref{lemma:2} is a direct consequence of the sequential behavior composition of FTS actions where each action has effects within the domain of attraction of the action below. The asymptotically stable equilibrium of each controller is either the goal $\bm C$, or it is within the region of attraction of another action earlier in the sequence, see \cite{Colledanchise2017,Burridge1999, Najafi2015}. One can visualize the idea of sequential composition in Fig.~\ref{fig:funnels}.

\begin{figure}[!htb]
    \centering
    \includegraphics[width=0.43\textwidth]{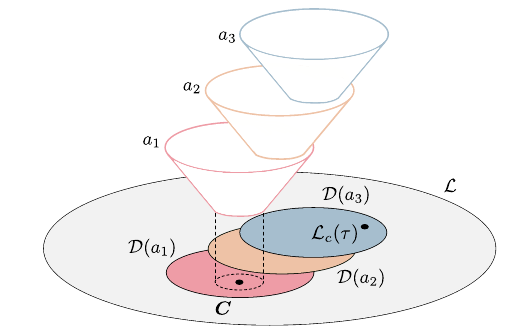}
    \caption{Schematic visualization of the domain of attraction $\mathcal{D}(\cdot)$ of different controllers around the current logical state $\mathcal{L}_c(\tau)$, as well as their postconditions within the domain of attraction of the controller below.}
    \label{fig:funnels}
\end{figure}

\subsection{Analysis of robustness}
It is not easy to find a common and objective definition of robustness to analyze the characteristics of algorithms for task execution. One possible way is to describe robustness in terms of domains or regions of attraction as in past work \cite{Burridge1999,Colledanchise2017}. When considering task planning and execution with classical BTs, often these regions of attraction are defined offline leading to a complex and extensive analysis of the possible contingencies that might eventually happen \cite{Colledanchise2017}, and these are specific to each different task. Alternatively, adapting the region of attraction requires either re-planning \cite{Paxton2019} or dynamic BT expansion \cite{Colledanchise2019}. Robustness can be measured according to the size of this region, such that a robot can achieve the desired goal from a plurality of initial conditions.
With Algorithm~\ref{alg:idea} we achieve robust behavior by \textit{dynamically} generating a suitable region of attraction according to the minimization of free-energy. This region {brings the current state towards the desired goal.} We then cover only the necessary region in order to be able to steer the current state to the desired goal, changing prior preferences at run-time. 

\begin{corollary}
\label{cor:robustness}
{When an executable action $a_N$ is found during task execution through Algorithm~2 such that $\pi = \{a_1,\dots,a_N\}$, the plan has a domain of attraction towards a given goal that includes the current state $\mathcal{L}_c(\tau)$. If $\mathcal{D}(a_1) \subseteq \bigcup_{a_i}\mathcal{P}_c(a_i)$ for $i = 2, ..., N$ the plan is asymptotically stable.}
\end{corollary} 

\begin{proof}
The corollary follows simply from \textit{Lemma} \ref{lemma:2}.
\end{proof}

\begin{example} \textit{Let us assume that Algorithm~\ref{alg:idea} produced a plan $\pi=\{a_1,\ a_2\}$, a set of FTS actions, where $a_2$ is executable so {$\mathcal{L}_c(\tau)\in\mathcal{D}(a_2)$}, and its effects are such that $\mathcal{P}_c(a_2)\equiv \mathcal{D}(a_1)$. Since $a_2$ is FTS, after a certain running time {$\mathcal{L}_c(\tau) \in \mathcal{P}_c(a_2)$}. The next tick after the effects of $a_2$ took place, $\pi = \{a_1\}$ where this time $a_1$ is executable since {$\mathcal{L}_c(\tau)\in\mathcal{D}(a_1)$} and $\mathcal{P}_c(a_1)= \bm C$. The overall goal is then achieved in a finite time.}
\end{example}

Instead of looking for globally asymptotically stable plans from each initial state to each possible goal, which can be unfeasible or at least very hard \cite{Burridge1999}, we define smaller regions of attractions dynamically, according to the current state and goal.

%% file: tex/experiments.tex
\section{Experimental evaluation}
\label{sec:Experiments}
In this section, we evaluate our algorithm in terms of robustness, safety, and conflicts resolution in two validation scenarios with two different mobile manipulators and tasks. We also provide a theoretical comparison with classical and dynamically expanded BTs.

\subsection{Experimental scenarios}
\subsubsection{Scenario 1} The task is to pick one object from a crate and place it on top of a table. This object might or might not be reachable from the initial robot configuration, and the placing location might or might not be occupied by another movable box. Crucially, the state of the table cannot be observed until the table is reached. This results in a partially observable initial state at the start of the mission where the place location has a 50\% chance of being either free or occupied. Additionally, we suppose that other external events, or agents, can interfere with the execution of the task, resulting in either helping or adversarial behavior. 
The robot used for the first validation scenario is a mobile manipulator consisting of a Clearpath Boxer mobile base, in combination with a Franka Emika Panda arm. The experiment for this scenario was conducted in a Gazebo simulation in a simplified version of a real retail store, see Fig.~\ref{fig:gazebo}.
\begin{figure}[!htb]
    \centering
    \includegraphics[width=0.35\textwidth]{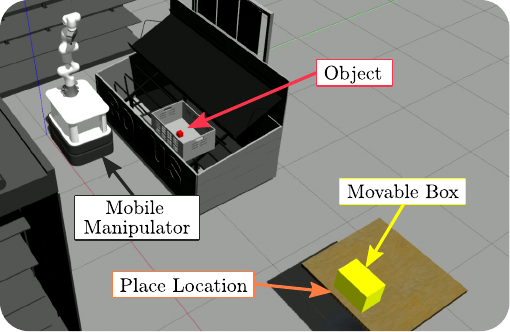}
    \caption{Simulation of the mobile manipulation task.}
    \label{fig:gazebo}
\end{figure}

\subsubsection{Scenario 2} The task is to fetch a product in a mockup retail store and stock it on a shelf using the real mobile manipulator TIAGo, as in Fig.~\ref{fig:tiago}.

Importantly, the BT for completing the task in the real store with TIAGo is the same one used for simulation with the Panda arm and the mobile base, just parametrized with a different object and place location. The code developed for the experiments and theoretical examples is publicly available.\footnote{\url{https://github.com/cpezzato/discrete_active_inference}}.
\begin{figure}[!htb]
    \centering
    \includegraphics[width=0.35\textwidth]{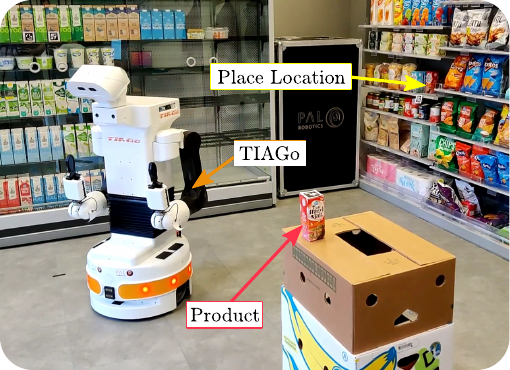}
    \caption{Experiments with TIAGo, stocking a product on the shelf.}
    \label{fig:tiago}
\end{figure}

\subsection{Implementation}

\subsubsection{Models for Scenarios 1 and 2}
In order to program the tasks for\textit{ Scenarios 1} and \textit{2}, we extended  the robot skills defined in our theoretical Example~\ref{ex:movePick}. We added then two extra states and their relative observations: \texttt{isPlacedAt(loc, obj)} called ${s^{(place)}}$, and \texttt{isLocationFree(loc)} called ${s^{(free)}}$. The state ${s^{(place)}}$ indicates whether or not \texttt{obj} is at \texttt{loc}, with associated probability, while ${s^{(free)}}$ indicates whether \texttt{loc} is occupied by another object. Then, we also had to add three more actions, which are 1) \texttt{place(obj, loc)}, 2) \texttt{push(obj)} to free a placing location, and 3) \texttt{placeOnPlate(obj)}, to place the object held by the gripper on the robot's plate. We summarise states and skills for the mobile manipulator in Table~\ref{tab:States}.

\begin{table}[ht!]
\caption{Notation for states and actions}
\centering 
\begin{tabular}{C{2.8cm} C{4.7cm}} %
\hline\hline 
\textbf{State, Boolean State} & \textbf{Description}\\ [0.5ex] %
\hline 
${s^{(loc)}},\ {l^{(loc)}}$ & Belief about being at the goal location \\
${s^{(hold)}},\ {l^{(hold)}}$ & Belief about holding an object\\
${s^{(reach)}},\ {l^{(reach)}}$ & Belief about reachability of an object\\
${s^{(place)}},\ {l^{(place)}}$ & Belief about an object being placed at a location\\
${s^{(free)}},\ {l^{(free)}}$ & Belief about a location being free\\

\hline 
\end{tabular}
\label{tab:States}
\end{table}

\begin{table}[h]
\centering 
\begin{tabular}{l l l} 
    \textbf{Actions}& \textbf{Preconditions}& \textbf{Postconditions}\\ 
    \texttt{moveTo(goal)} & \texttt{-} & ${l^{(loc)}} = [1\ 0]^\top$ \\
    & & ${l^{(reach)}} = [1\ 0]^\top$ \\    
    \texttt{pick(obj)} & \texttt{isReachable(obj)}& ${l^{(hold)}} = [1\ 0]^\top$\\ 
     & \texttt{!isHolding} & \\
    \texttt{place(obj,loc)} & \texttt{isLocationFree(loc)}& ${l^{(place)}} = [1\ 0]^\top$\\
    \texttt{push()} & \texttt{!isHolding} & ${l^{(free)}} = [1\ 0]^\top$\\
    \texttt{placeOnPlate()} & - & ${l^{(hold)}} = [0\ 1]^\top$\\
\end{tabular}
\end{table}

The likelihood matrices are just the identity, while the transition matrices simply map the postconditions of actions, similarly to Example~\ref{ex:movePick}. Note that the design of actions and states is not unique, and other combinations are possible. One can make atomic actions increasingly more complex, or add more preconditions. The plan, specified in a BT, contains the desired sequence of states to complete the task, leaving out from the offline planning other complex fallbacks to cope with contingencies associated with the dynamic nature of the environment. The BT for performing the tasks in both \textit{Scenario~1} and \textit{Scenario~2} is reported in Fig.~\ref{fig:bt_experimets}. Note that the fallback for the action \texttt{moveTo} could be substituted by another prior node as in Fig.~\ref{fig:btSimple}, however, we opted for this alternative solution to highlight the hybrid combination of classical BTs and active inference. Design principles to choose when to use prior nodes and when normal fallbacks are reported in Sec.\ref{sec:comparison}.

\begin{figure}[!htb]
    \centering
    \includegraphics[width=0.42\textwidth]{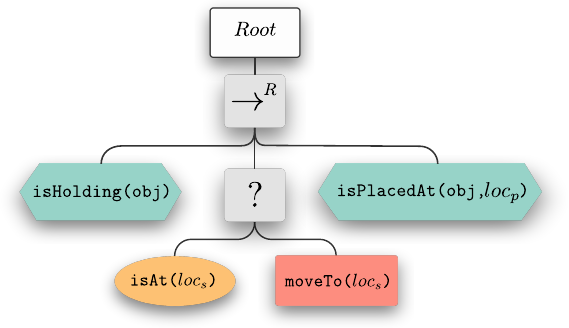}
    \caption{BT with prior nodes to complete the mobile manipulation task in the retail store, \textit{Scenario 1} and 2. $loc_s$, $loc_p$ are respectively the location in front of the shelf in the store and the desired place location of an item}
    \label{fig:bt_experimets}
\end{figure}

\subsubsection{Execution of Algorithm~\ref{alg:idea}}
{We provide a full execution of Algorithm~\ref{alg:idea} in Example~\ref{ex:alg_steps}. We consider \textit{Scenario~1,} for which the initial configuration of the robot is depicted in Fig.~\ref{fig:gazebo}, and the BT for the task is the one if Fig.~\ref{fig:bt_experimets}.}

{\begin{example}
\label{ex:alg_steps}
Let us consider Algorithm~\ref{alg:idea} and Scenario~1. At the start of the task, the first node \texttt{isHolding(obj)} in the BT is ticked, and the corresponding prior preference is set, $\bm C^{(hold)} = [1, 0]^\top $ (line 2 Alg.~\ref{alg:idea}). Since the robot is not holding the desired \texttt{obj} and it is not reachable, $o^{(hold)}=o^{(reach)}=[0, 1]^\top$. At the start, the initial states are a uniform distribution $s^{(hold)}=s^{(reach)}=[0.5, 0.5]^\top$ (line 4). Since the task just started, the only prior preference is the one set by the BT, so there are no high-priority priors (lines 6-10). Algorithm~\ref{alg:AIC} is then run (line 12), updating the states $s^{(hold)},\ s^{(reach)}$ according to the given observations, and selecting an action $a_\tau$. In this example, the updated most probable logical state (line 13) will be $l^{(hold)}_c=l^{(reach)}_c=[0, 1]^\top$ and $a_\tau$ will be \texttt{pick(obj)} since there is a mismatch between the $C^{(hold)}$ and $l^{(hold)}_c$. The preconditions of \texttt{pick(obj)} are checked (line 19). This action requires the object to be reachable, so this missing precondition is added to the preferences with high priority, that is $C^{(reach)}=[2, 0]^\top$. Active inference is run again with the update prior (line 27). This process (lines 17-32) is repeated until either an executable action is found (lines 19-21) or the selected action is \texttt{idle} (lines 28-30). In the first case, $a_\tau$ is executed and the algorithm returns running. In the second case, a failure is returned indicating that no action can be performed to satisfy prior preferences. 
In this example, re-running active inference (line 27)  with $\bm C^{(hold)} = [1, 0]^\top$ and $\bm C^{(reach)} = [2, 0]^\top$  would return the action \texttt{moveTo}. There are no preconditions for this action, thus it can be executed (lines 20-21). The BT keeps being ticked at a certain frequency, and until the object becomes reachable, the same steps as just described will be repeated. As soon as the robot can reach the object, so $l^{(reach)}_c = [1, 0]^\top$, the preference for $\bm C^{(reach)}$ is removed (lines 6-10) and set to $[0, 0]^\top$. This time, when the action \texttt{pick} is selected, its preconditions are satisfied and thus it can be executed. After holding the object, when the BT is ticked no action is needed since the prior is already satisfied. Algorithm~\ref{alg:idea} returns success (lines 14-16) and the task can proceed with the next node in the BT.
\end{example}}
Note that the BT designed for \textit{Scenario 1} was entirely reused in \textit{Scenario 2}, with the only adaptation of the desired object and locations in the BT. In Sec.~\ref{sec:robustness} and Sec.~\ref{sec:runtimeconflicts}, robustness and run-time conflicts resolution are analyzed for \textit{Scenario 1}, but similar considerations can be derived for \textit{scenario 2}.

\subsection{Robustness: Dynamic regions of attraction}\label{sec:robustness}

With our approach we improve robustness compared to classical BTs in two different ways: 1) in terms of task execution, and 2) against noisy sensory readings. We elaborate on the following:

\subsubsection{Robustness of task execution}
With our algorithm we can generate complex regions of attraction dynamically at runtime, alleviating the burden of programming every fallback beforehand in a BT, which is prone to fail in edge cases that haven't been considered at design time. We illustrate this in the example below. Consider Example~\ref{ex:alg_steps}. According to the current world's state, Algorithm~\ref{alg:idea} selects different actions to generate a suitable domain of attraction:
\begin{example}
\label{ex:doa}
\textit{The initial conditions are such that the object is not reachable. Let ${s^{(hold)}}$ be the probabilistic belief of holding an object, and ${s^{(reach)}}$ be the probabilistic belief of its reachability. The domain of attraction generated by Algorithm~\ref{alg:idea} at runtime is depicted in Fig.~\ref{fig:bt_pp} using phase portraits as in \cite{Colledanchise2017}. Actions, when performed, increase the probability of their postconditions.}

\begin{figure}[htb]
    \centering
    \includegraphics[width=0.44\textwidth]{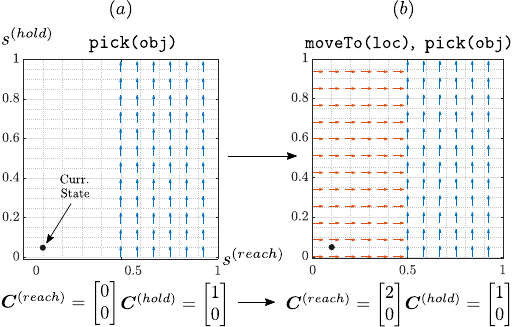}
    \caption{Dynamic domain of attraction generated by Algorithm \ref{alg:idea} for Example~\ref{ex:doa}. (a) relates to the action \texttt{pick(obj)}, and (b) is the composition of \texttt{moveTo(loc)} and \texttt{pick(obj)} after automatically updating the prior preferences}
    \label{fig:bt_pp}
\end{figure}
\end{example}
From Fig.~\ref{fig:bt_pp}, we can see that the goal of the active inference agent is to hold \texttt{obj} so ${\bm C^{(hold)}} = [1\ 0]^\top$. The first selected action is then \texttt{pick(obj)}. However, since the current logical state is not contained in the domain of attraction of the action, the prior preferences are updated with the missing (higher priority) precondition according to the action template provided, that is \texttt{isReachable} so ${\bm C^{(reach)}} = [2, 0]^\top$. This results in a sequential composition of controllers with a stable equilibrium corresponding to the postconditions of \texttt{pick(obj)}. On the other hand, to achieve the same domain of attraction with a classical BT, one would require several additional nodes, as explained in Sec.~\ref{sec:comparison} and visualized in Fig.~\ref{fig:bt_comparison}. Instead of extensively programming fallback behaviors, Algorithm~\ref{alg:idea} endows our actor with deliberation capabilities and allows the agent to reason about the current state of the environment, the desired state, and the available action templates.

\subsubsection{Robustness against noisy sensory readings}
BTs are purely reactive which means that every action, or sub-behavior, is executed in response to an event or a condition determined at the current time step $\tau$. If an instantaneous observation is erroneous, wrong transitions could be triggered because in classical BTs there is no notion and representation of a ``state" that is maintained and updated over time. This might be a problem in the presence of noise in the observations.
\begin{example}
\label{ex:noise}
Consider a generic fallback based on a condition check in a BT, as in Fig.~\ref{fig:robustness_noise}. The perception system, on which the condition check is based produced at time $\tau$ an erroneous observation, for instance, due to poor lighting conditions in an object detection algorithm. If the condition is checked at time $\tau$, a purely reactive system would produce a wrong transition because the condition returns failure, and the fallback would tick the action. 
\end{example}
\begin{figure}[!htb]
    \centering
    \includegraphics[width=0.4\textwidth]{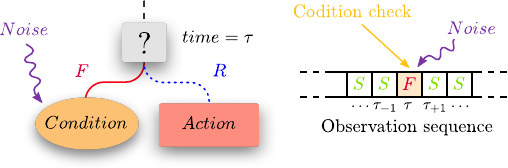}
    \caption{Erroneous transition in a purely reactive BT due to a noisy observation at time $\tau$. $F$, $S$, and $R$ mean respectively, \textit{Failure, } \textit{Success} and \textit{Running}.}
    \label{fig:robustness_noise}
\end{figure}

In active inference, the probabilistic representation of a state helps filter out this kind of spurious sensory input. For instance, a wrong observation like the one above would have caused the robot to be just slightly less confident about that state, with no erroneous transition.

\subsection{Resolving run-time conflicts}
\label{sec:runtimeconflicts}
Unexpected events affecting the system during action selection can lead to conflicts with the initial offline plan. This is the case in one execution  of the mobile manipulation task as in Fig.~\ref{fig:panda_conflicts} where after picking the object and moving in front of the table, the robot senses that the place location is not free. In this situation, a conflict with the offline plan arises, {where there is a preference for two mutually exclusive states, namely holding the red cube but also having the gripper free in order to empty the place location. We describe this situation more formally in Example~\ref{ex:conflict}, and we then explain how such a conflict is resolved.}
\begin{example}
\label{ex:conflict}\textit{
For solving the situation in Fig.~\ref{fig:panda_conflicts} using the BT in Fig.~\ref{fig:bt_experimets}, the robot should 1) hold the object, 2) be at the desired place location, and 3) have the object placed. The preferences are planned offline and the BT populates the relative priors with a unitary preference at runtime (eq.~\eqref{eq:priors}). At this point of the execution, there is a mismatch between the current logical belief about the state ${l^{(place)}}$ and the desired {one, in fact ${\bm C^{(place)}}$ differs from ${l^{(place)}}$ (eq.~\eqref{eq:logicalStates}) because the object is not placed at the desired location.}}
\begin{subequations}
\begin{equation}
\label{eq:priors}
    {\bm C^{(hold)}} = \begin{bmatrix} 1 \\ 0
    \end{bmatrix},\ {\bm C^{(place)}} = \begin{bmatrix} 1 \\ 0
    \end{bmatrix}
\end{equation}
\begin{equation}
\label{eq:logicalStates}
    {l^{(hold)}} = \begin{bmatrix} 1 \\ 0
    \end{bmatrix},\ {l^{(place)}} = \begin{bmatrix} {0} \\ {1}
    \end{bmatrix}
\end{equation}
\end{subequations}
{The selected action with active inference in this situation is \texttt{place(obj,loc)}. The missing precondition on the place location to be free is added to the prior (see ${\bm C^{(free)}}$ in eq.~\eqref{eq:prior_conflicts}) and the action selection is performed again. The only action that can minimize free-energy further is now \texttt{push}}. Then, the missing precondition for this action (i.e. \texttt{!isHolding}) is added in the current prior with higher priority {in ${\bm C^{(hold)}}$} (line 24 in Algorithm~\ref{alg:idea}). 
\begin{equation}
    \label{eq:prior_conflicts}
    {\bm C^{(free)}} = \begin{bmatrix} 2 \\ 0
    \end{bmatrix},\ {\bm C^{(hold)}} = \begin{bmatrix} 1 \\ \textcolor{tudelft-fuchsia}{2}
    \end{bmatrix} 
\end{equation}
\vspace{0mm}
\end{example}

{The required \texttt{push} action to proceed with the task has a conflicting precondition with the offline plan, see ${\bm C^{(hold)}}$ in eq.~\eqref{eq:prior_conflicts}. }
\begin{figure}[!htb]
    \centering
    \includegraphics[width=0.35\textwidth]{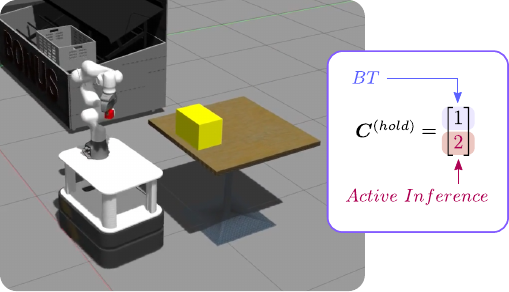}
    \caption{{Example of conflict during mobile manipulation using the BT from Fig.~\ref{fig:bt_experimets}. The BT is defining a preference of 1 over holding the red cube, but the runtime situation requires a free gripper to safely push away an unexpected object. Active inference is then required to achieve an unmet precondition with higher preference.}}
    \label{fig:panda_conflicts}
\end{figure}

Even though the desired state specified in the BT is \texttt{isHolding(obj)}, at this particular moment there is a higher preference for having the gripper free due to a missing precondition to proceed with the plan. Algorithm~\ref{alg:idea} selects then the action that best matches the current prior desires, or equivalently that minimises expected free-energy the most, that is \texttt{placeOnPlate} to obtain ${l^{(hold)}} = [0,\ 1]^\top$. This allows, then, to perform the action \texttt{push}. {Once the place location is free after pushing, the high-priority preference on having the location free is removed from the prior. As a consequence,} there are also no more preferences pushed with high priority over the state ${l^{(hold)}}$ which is only set by the BT as  ${\bm C^{(hold)}} = [1,\ 0]^\top$. The red cube is then picked again and placed on the table since no more conflicts are present. 

Videos of the simulations and experiments can be found online\footnote{\url{https://youtu.be/dEjXu-sD1SI}}. The BTs to encode priors for active inference are implemented using a standard library \cite{bt_library}.

\subsection{Safety}
When designing adaptive behaviors for autonomous robots, attention should be paid to \textit{safety}. The proposed algorithm allows to retain control over the general behavior of the robot and to force a specific routine in case something goes wrong leveraging the structure of BTs. In fact, we are able to include adaptation through active inference only in the parts of the task that require it, keeping all the properties of BTs intact. Safety guarantees for the whole task can easily be added by using a sequence node where the leftmost part is the safety criteria to be satisfied \cite{Colledanchise2017} {by the right part of the behavior}, as shown in Fig.~\ref{fig:safety}. 
\begin{figure}[!htb]
    \centering
    \includegraphics[width=0.35\textwidth]{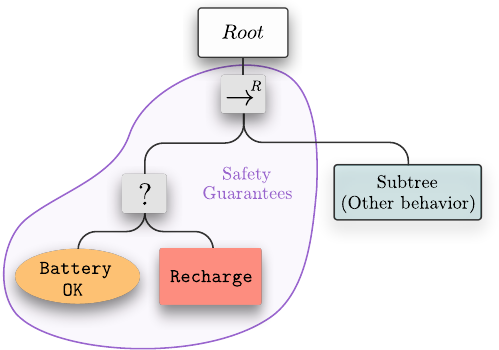}
    \caption{BT with safety guarantees while allowing runtime adaptation.}
    \label{fig:safety}
\end{figure}
In this example, the BT allows avoiding battery drops below a safety-critical value while performing a task. The sub-tree on the right can be any other BT, for instance, the one used to solve \textit{Scenario 1} and \textit{Scenario 2} from Fig.~\ref{fig:bt_experimets}. 

Since, by construction, a BT is executed from left to right, one can assure that  the robot is guaranteed to satisfy the leftmost condition first before proceeding with the rest of the behavior. In our specific case, this allows us to easily override the online decision making process with active inference where needed, in favor of safety routines. {Note that safety guarantees can also be provided in specific parts of the three only, and not necessarily for the whole tree. For example in Fig.~\ref{fig:bt_experimets}, one might ensure navigation at a low speed only while transporting an object to a place location}.

\subsection{Comparison and design principles}
\label{sec:comparison}
\begin{figure*}
    \centering
    \includegraphics[width=0.9\textwidth]{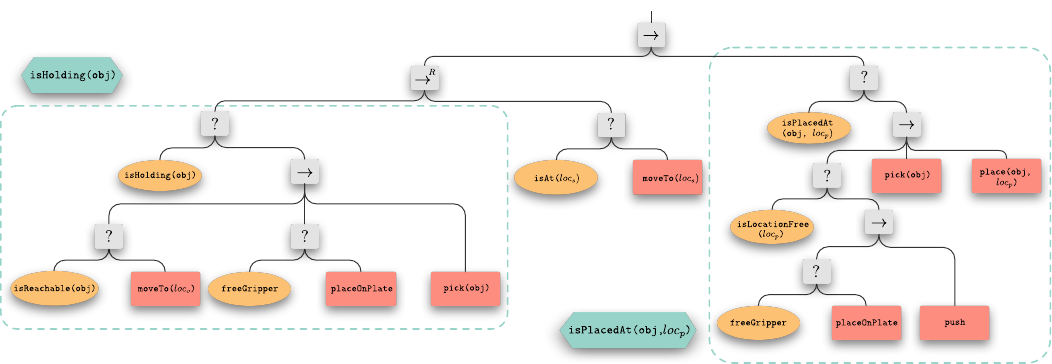}
    \caption{Possible standard BT to perform \textit{Scenario 1} and \textit{Scenario 2} without prior nodes for active inference. Parts of the behavior that require several fallbacks can be substituted by prior nodes for online adaptation instead.}
    \label{fig:bt_comparison}
\end{figure*}

\subsubsection{Comparison with other BT approaches}
The hybrid scheme of active inference and BTs aims at providing a framework for reactive action planning and execution in robotic systems. For this reason, we compare the properties of our approach with standard BTs \cite{Colledanchise2017} and with BTs generated through expansion from goal conditions \cite{Colledanchise2019}. 
\textit{Scenario 1} and \textit{Scenario 2} can be tackled, for instance, by explicitly planning every fallback behavior with classical BTs, as in Fig.~\ref{fig:bt_comparison}. Even if this provides the same reactive behavior as the one generated by Fig.~\ref{fig:bt_experimets}, far more (planning) effort is needed: to solve the same task one would require 12 control nodes, 8 condition nodes, and 7 actions, for a total of 27 nodes compared to the 6 needed in our approach that is an $\sim88\%$ compression.

Importantly, the development effort of a prior node in a BT is the same as a standard action node. It is true that active inference requires specifying the likelihood and transition matrices encoding actions pre- and postconditions, but this has to be done only once while defining the available skills of a robot, and it is independent of the task to be solved. Thus, a designer is not concerned with this when adding a prior node in a BT.

Instead of planning several fallbacks offline, \cite{Colledanchise2019} dynamically expands a BT from a single goal condition, through backchaining, to blend planning and acting online. To solve \textit{Scenario 1} and \textit{Scenario 2} with this approach, one needs to define a goal condition \texttt{isPlacedAt(obj, loc)} similarly to our solution, and define the preconditions of the action \texttt{place(obj, loc)} such that they contain the fact that the robot is holding the object, that the place location is reachable, and it is free. Then, to solve \textit{Scenarios 1} and \textit{2} one needs to define only the final goal condition and run the algorithm proposed in \cite{Colledanchise2019}. Even though this allows to complete tasks similar to what we propose, \cite{Colledanchise2019} comes with a fundamental theoretical limitation: adaptation cannot be selectively added only to specific parts of the tree. The whole behavior is indeed determined at runtime based on preconditions and effects of actions starting from a goal condition. {The addition of safety guarantees can only happen for the whole task, and not in selected parts of the tree derived online.} 

To conclude, the hybrid combination of active inference and BTs allows for combining the advantages of both offline design of BTs and online dynamic expansion. In particular: it drastically reduces the number of necessary nodes planned offline in a BT, it can handle partial observability of the initial state, and it allows to selectively add adaptation and safety guarantees in specific parts of the tree. 

Another important difference between our approach and other BTs solutions is that we introduced the concept of \textit{state} in a BT through the \textit{prior} node for which a probabilistic belief is built, updated, and used for action planning at runtime with uncertain action outcomes.  

Table~\ref{tab:comparison} reports a summary of the comparison with standard BTs and BT with dynamic expansion for \textit{Scenarios 1} and \textit{2}. 

\begin{table}[ht]
\caption{Summary of comparison} \centering 
\begin{tabular}{C{1.2cm} C{1cm} C{1.5cm} C{1.5cm} C{1.5cm}} %
\hline\hline 
 \textbf{Approach} & \textbf{\# {Hand-crafted nodes}}& \textbf{Unforeseen contingen.} & \textbf{Selective adaptation} & {\textbf{Adaptive safety guarantees}} \\ [0.5ex] 
\hline 
Standard BT & 27 & \xmark & \xmark & \cmark\\
\hline
Dynamic BT \cite{Colledanchise2019} & 1 & \cmark & \xmark & {Only for the task as a whole} \\ 
\hline
Ours & 6 & \cmark & \cmark & \cmark \\ 
\hline
\end{tabular}
\label{tab:comparison}
\end{table}

\subsubsection{Comparison with ROSPlan}
One may argue that other solutions such as ROSPlan \cite{cashmore2015rosplan} could also be used for planning and execution in robotics in dynamic environments. ROSPlan leverages automated planning with PDDL2.1, but it is not designed for fast reaction in case of dynamic changes in the environment caused by external events, {and would not work in case of a partially observable initial state}. {Consider \textit{Scenario~1}. At the start of the task, the robot can sense that the gripper is empty and it has access to its current base location. However, it has no information about the state of the table (occupied or free) on which the red cube needs to be placed. At the start of the mission, the table has then 50\% chance of being occupied and a 50\% chance of being free, since it cannot be observed. Yet, the whole task can be planned at a high level, as in Fig. \ref{fig:bt_experimets}, and be executed. Once the robot reaches the placing location, observations regarding the state of the table become available and the internal beliefs can be updated until enough evidence is collected and a decision can be taken. Without knowing the full state at the start, a solution with ROSPlan would require making assumptions on the value of the unknown states, and either planning for the worst-case scenario, which might not even be needed, or failing during execution and re-plan.} 

\subsubsection{Design principles}
We position our work in between two extremes, namely fully offline planning and fully online dynamic expansion of BTs. 
In our method, a designer can decide if to lean towards a fully offline approach or a fully online synthesis. The choice depends  on the task at hand and the modeling of the actions pre- and postconditions. Even though the design of behaviors is still an art, we give some design principles which can be useful in the development of robotic applications using this hybrid BTs and active inference method. Take for instance Fig.~\ref{fig:bt_comparison} and Fig.~\ref{fig:bt_experimets}. Prior nodes for local adaptation can be included in the behavior when there are several contingencies to consider or action preconditions to be satisfied in order to achieve a sub-goal. A designer can: 1) plan offline where the task is certain or equivalently where a small number of things can go wrong; 2) use prior nodes implemented with active inference to decide at runtime the actions to be executed whenever the task is uncertain. This is a compromise between a fully defined plan where the behavior of the robot is predefined in every part of the state space and a fully dynamic expansion of BTs which can result in a sub-optimal action sequence \cite{Colledanchise2019}. This is illustrated in Fig.\ref{fig:bt_experimets}, where the actions for holding and placing an object are chosen online due to various possible unexpected contingencies, whereas the \texttt{moveTo} action is planned. Prior nodes should be used whenever capturing the variability of a part of a certain task would require much effort during offline planning. 

%% file: tex/discussion.tex
\section{Discussion}
\label{sec:Discussion}
{In this work, we considered a mobile manipulator in a retail store domain  with particular focus on plan execution. In this scope, offline planning is arguably better suited to offload computations at runtime for the parts of the task that do not change frequently. In our case, this was the sequence of states to stock a product that was encoded in a BT. On the other hand, at the cost of additional online computations, local online planning is better suited for execution in uncertain environments such as a busy supermarket, because one can avoid planning beforehand for every contingency. We achieved this through active inference. With the proposed method we can leverage the complementary advantages of both offline and online planning, and in the following we discuss in detail the choice of using BTs and active inference specifically.}

\subsection{Active inference as a planning node in a BT}
We opted for the use of active inference with action preconditions as a planning node in the BT because this allows achieving online PDDL-style planning with noisy observations and partially observable probabilistic states. An alternative solution to our approach could be to use a simple PDDL planner in conjunction with a filtering scheme. By re-planning for a small sub-task at the same frequency used for the active inference node, one can achieve similar reactiveness to our approach. However, by means of active inference, one can make use of the full probabilistic information on the states, and one does not require full knowledge of the symbolic initial state.

{\subsection{Why choosing BTs}
An experienced roboticist could also wonder why we opted for BTs in the first place, instead of other PDDL-style planning approaches to encode the solution to a task. First of all, the focus of this paper is on the runtime adaptability and reactivity in dynamic environments with partially observable initial state, and not on the generation of complex offline plans.}

{In this context, BTs are advantageous because they are designed to dispatch and monitor actions execution at runtime. This allows to quickly react to changes in the environment that are not necessarily a consequence of the robot's own actions. A plan that is generated through PDDL planning and executed bypassing BTs cannot provide this type of reactivity unless one defines specific re-planning strategies to mimic BTs' reactiveness.}

{However, we see potential extensions of this work by combining it with other PDDL planning methods. For instance, PDDL can be used to automatically generate plans for which their execution is optimized through BTs \cite{martin2021}. This allows for instance to take advantage of parallel action execution and would remove the need to hand-design a BT. Additionally, an action at runtime can be executed as soon as its requirements are available instead of waiting for what is established offline by the planner.}
\subsection{Why choose active inference}
Active inference could potentially be substituted by other valid POMDP approaches. We see two possible ways of doing so, that is either using an \textit{offline} or an \textit{online} POMDP solver.

First, one could {solve} a policy \textit{offline} and then use it for online decision making. This approach {can be more effective than active inference} once the transition matrices, the reward, and the task are fixed. However, the addition of new symbolic actions (so new skills as transitions), or a substantial change in the task while using the same skills, would require re-computing the policy. {In addition, planning offline for all possible states and actions combinations is a much larger problem than computing a plan online for the current state only}. As concluded in \cite{moSSP}, {for the parts of a task subject to frequent unpredictable changes}, computing {locally} an online plan as we do with active inference is preferable to offline policies.

Second, one could perform \textit{online} decision making without offline {computations and achieve similar performance}, see \cite{paquet2005online, ye2017despot}. {Compared to both offline and online POMDPs, however}, active inference exposes extra model parameters to bridge abstract common sense knowledge with discrete decision making. {These extra model parameters can be updated with runtime information to adapt plans on the fly}. Take for instance the prior over plans $p(\pi) = Cat(\bm E)$. The $\bm E$ vector can be updated at runtime to steer the decision making while computing the posterior distribution $\bm \pi = \sigma(\ln\bm E -\bm G_\pi - \bm F_\pi)$ \cite{smith2021step}. This vector can be used to {encode common sense and habits \cite{smith2021step, hesp2021}, but can also be used to adapt the plans online due for instance to runtime component failure.} This opens up many possibilities for extensions, to be explored in future work. We are particularly interested in active inference because it is a flexible and unified framework that connects different branches of control theory at different abstraction levels.
Active inference can unify (i) abstract decision making with guarantees, as in this paper, (ii) adaptive \cite{pezzato2020novel} and fault tolerant \cite{pezzato2020active, baioumy2021fault, baioumy2021towards} torque control, as well as (iii) state estimation and learning.

%% file: tex/conclusions.tex
\section{Conclusions}
\label{sec:Conclusions}
In this work, we tackled the problem of action planning and execution in real-world robotics. We addressed two open challenges, namely hierarchical deliberation, and continual online planning, by combining BTs and active inference.
The proposed algorithm and its core idea are general and independent of the particular robot platform and task. Our solution provides local reactivity to unforeseen situations while keeping the initial plan intact. In addition, it is possible to easily add safety guarantees to override the online decision making process thanks to the properties of BTs. We showed how robotic tasks can be described in terms of free-energy minimization, and we introduced action preconditions and conflict resolution for active inference by means of dynamic priors. This means that a robot can locally set its own sub-goals to resolve a local inconsistency, and then return to the initial plan specified in the BT. We performed a theoretical analysis of the convergence and robustness of the algorithm, and the effectiveness of the approach is demonstrated on two different mobile manipulators and different tasks, both in simulation and real experiments.

%% file: tex/appendix.tex
\section{Generative models}
\label{sec:AppendixP}
{Consider the generative model in active inference $P(\bar{o},\bar{s},\bm \eta,\pi)$. By using the chain rule, we can write:}
\begin{equation}
   P(\bar{o},\bar{s},\bm \eta,\pi) = P(\bar{o}|\bar{s},\bm \eta,\pi)P(\bar{s}|\bm \eta,\pi)P(\bm \eta|\pi)P(\pi)
\end{equation}
Note that $\bar{o}$ is conditionally independent from the model parameters $\bm \eta$ and $\pi$ given $\bar{s}$. In addition, under the Markov property, the next state and current observations depend only on the current state:
\begin{equation}
     P(\bar{o}|\bar{s},\bm \eta, \pi) = \prod_{\tau=1}^T{P(o_\tau|s_\tau)}
\end{equation}
The model is further simplified considering that $\bar{s}$ and $\bm \eta$ are conditionally independent given $\pi$: 
\begin{equation}
    P(\bar{s}|\bm \eta,\pi) = \prod_{\tau=1}^T{P(s_\tau|s_{\tau-1}, \pi)}
\end{equation}
Finally, consider the model parameters explicitly:
\begin{align}
    \nonumber
    &P(\bar{o},\bar{s},\bm \eta,\pi) = P(\bar{o},\bar{s},\bm A,\bm B,\bm D,\pi) =  &\\
    & P(\pi)P(\bm A)P(\bm B)P(\bm D)\prod_{\tau=1}^T P(s_\tau|s_{\tau-1},\pi)P(o_\tau|s_\tau) & 
\end{align}
{$P(\bm A), P(\bm B), P(\bm D)$ are Dirichlet distributions over the model parameters, \cite{friston2017active}. In case the model parameters are fixed by the user, as in this work, it holds:
\begin{equation}
    P(\bar{o},\bar{s},\pi) = P(\pi)\prod_{\tau=1}^T P(s_\tau|s_{\tau-1},\pi)P(o_\tau|s_\tau)
    \label{eq:model_fixed_par}
\end{equation}
}

Given the generative model above, we are interested in finding the posterior hidden causes of sensory data. For the sake of these derivations, we consider that the parameters associated with the task are known and do not introduce uncertainty. Using Bayes rule:
\begin{equation}
    \label{eq:app_bayes}
    P(\bar{s},\pi|\bar{o}) = \frac{P(\bar{o}|\bar{s},\pi)P(\bar{s},\pi)}{P(\bar{o})} 
\end{equation}
Computing the model evidence $P(\bar{o})$ exactly is a well-known and often intractable problem in Bayesian statistics. The exact posterior is then computed minimizing the Kullback-Leibler divergence ($D_{KL}$, or KL-Divergence) with respect to an approximate posterior distribution $Q(\bar{s},\pi)$. Doing so, we can define the free-energy as a functional of approximate posterior beliefs which result in an upper bound on surprise. By definition $D_{KL}$ is a non-negative quantity given by the expectation of the logarithmic difference between $Q(\bar{s},\pi)$ and $P(\bar{s},\pi|\bar{o})$. Applying the KL-Divergence:
\begin{align}
\label{eq:app_dkl}
\nonumber
    &D_{KL}\left[ Q(\bar{s},\pi)||P(\bar{s},\pi|\bar{o}) \right] = & \\
    &\mathbb{E}_{Q(\bar{s},\pi)}\left[\ln{Q(\bar{s},\pi)}-\ln{P(\bar{s},\pi|\bar{o})}\right]\geq 0&
\end{align}
$D_{KL}$ is the information loss when $Q$ is used instead of $P$. Considering equation \eqref{eq:app_bayes} and the chain rule, equation \eqref{eq:app_dkl} can be rewritten as:
\begin{align}
\nonumber
    D_{KL}\left[\cdot\right] &= \mathbb{E}_{Q(\bar{s},\pi)}\left[    \ln{Q(\bar{s},\pi)}-\ln{\frac{P(\bar{o},\bar{s},\pi)}{P(\bar{o})}}
\right]\\
    \label{eq:app_Fupper}
     &= \underbrace{\mathbb{E}_{Q(\bar{s},\pi)}\left[\ln{Q(\bar{s},\pi)}-\ln{P(\bar{o},\bar{s},\pi)}\right]}_{F\left[Q(\bar{s},\pi)\right]}+\ln{P(\bar{o})}
\end{align}
We have just defined the free-energy as the upper bound of surprise:
\begin{equation}
    F\left[Q(\bar{s},\pi)\right] \geq - \ln{P(\bar{o})}
\end{equation}

\section{Variational Free-energy}
\label{sec:AppendixF}
To fully characterize the free-energy in equation \eqref{eq:app_Fupper}, we need to specify a form for the approximate posterior $Q(\bar{s},\pi)$. There are different ways to choose a family of probability distributions \cite{schwobel2018active}, compromising between complexity and accuracy of the approximation. In this work, we choose the mean-field approximation. It holds:
\begin{equation}
    Q(\bar{s},\pi) = Q(\bar{s}|\pi)Q(\pi) = Q(\pi)\prod_{\tau=1}^TQ(s_\tau|\pi)
\end{equation}
Under mean-field approximation, the plan-dependent states at each time step are approximately independent of the states at any other time step. We can now find an expression for the variational free-energy. Considering the mean-field approximation {and the generative model in eq.~\eqref{eq:model_fixed_par}} we can write:
\begin{eqnarray}
    \nonumber
    F\left[Q(\bar{s},\pi)\right] = \mathbb{E}_{Q(\bar{s},\pi)}\bigg[\ln{Q(\pi)} + \sum_{\tau=1}^T\ln{Q(s_\tau|\pi)}\\
    -\ln{P(\pi)} -\sum_{\tau=1}^T\ln{P(s_\tau|s_{\tau-1},\pi)} -\sum_{\tau=1}^T\ln{P(o_\tau|s_\tau)}\bigg]
\end{eqnarray}
Since $Q(\bar{s},\pi) = Q(\bar{s}|\pi)Q(\pi)$, and since the expectation of a sum is the sum of the expectation, we can write:
\begin{equation}
\label{eq:app_Ftot}
    F\left[\cdot\right] = D_{KL}\left[Q(\pi)||P(\pi)\right] +  \mathbb{E}_{Q(\pi)}\left[F(\pi)\left[Q(\bar{s}|\pi)\right] \right]
\end{equation}
where
\begin{align}
    \label{eq:app_Fpi}
    \nonumber
    & F(\pi)\left[Q(\bar{s}|\pi)\right] = \mathbb{E}_{Q(\bar{s}|\pi)}\bigg[\sum_{\tau=1}^T\ln{Q(s_\tau|\pi)}&\\
    &-\sum_{\tau=1}^T\ln{P(s_\tau|s_{t-\tau}, \pi)} -\sum_{\tau=1}^T\ln{P(o_\tau|s_\tau)}\bigg]&
\end{align}
One can notice that $F(\pi)$ is accumulated over time, or in other words, it is the sum of free energies over time and plans:
\begin{equation}
\label{eq:app_Fpitau}
    F(\pi) = \sum_{\tau=1}^TF(\pi,\tau)
\end{equation}
Substituting the agent's belief about the current state at time $\tau$ given $\pi$ with $\bm s_\tau^\pi$, we obtain a matrix form for $F(\pi,\tau)$ that we can compute given the generative model:
\begin{eqnarray}
    F(\pi) = \sum_{\tau=1}^T \bm s_\tau^{\pi\top}\bigg[\ln{\bm s_\tau^\pi} - \ln{{(}\bm B_{a_{\tau-1}}\bm s_{\tau-1}^{\pi}}{)} - \ln{{(}\bm A^\top \bm o_\tau{)}} \bigg]
\end{eqnarray}
Given a plan $\pi$, the probability of state transition $P(s_\tau|s_{\tau-1},\pi)$ is given by the transition matrix under plan $\pi$ at time $\tau$, multiplied by the probability of the state at the previous time step. In the special case of $\tau=1$, we can write:
\begin{equation}
    F(\pi,1) =  \bm s_1^{\pi\top} \big[\ln{\bm s_1^\pi} - \ln{\bm D} - \ln{{(}\bm A^\top \bm o_1{)}} \big]
\end{equation}
Finally, we can compute the expectation of the plan dependant variational free-energy $F(\pi)$ as $ \mathbb{E}_{Q(\pi)} \big[F(\pi)\big] = \bm\pi^\top \bm F_{\pi}$. We indicate $\bm F_{\pi} = (F(\pi_1),F(\pi_2)...)^\top$ for every allowable plan. To derive state and plan updates that minimize free-energy, $F$ in equation \eqref{eq:app_Ftot} is partially differentiated and set to zero, as we will see in the next appendixes. 
\section{State estimation}
\label{sec:AppendixS}
We differentiate $F$ with respect to the sufficient statistics of the probability distribution of the states. Note that the only part of $F$ dependent on the states is $F(\pi)$. Then:
\begin{align}
    \nonumber
    &\frac{\partial F}{\partial\bm s_{\tau}^{\pi}} = \frac{\partial F}{\partial F(\pi)}\frac{\partial F(\pi)}{\partial\bm s_{\tau}^{\pi}} = \bm\pi^\top \big[ \bm 1 + \ln{\bm s_\tau^\pi} 
    - \ln{{(}\bm B_{a_{\tau-1}}\bm s_{\tau-1}^{\pi}}{)} & \\
    &- \ln{{(}\bm B^\top_{a_{\tau}} \bm s_{\tau+1}^{\pi}{)}}
    - \ln{{(}\bm A^\top \bm o_\tau{)}} \big]&
\end{align}
Setting the gradient to zero and using the softmax function for normalization:
\begin{align}
    \bm s_{\tau}^{\pi} = \sigma(\ln{{(}\bm B_{a_{\tau-1}}\bm s_{\tau-1}^{\pi}}{)} 
    + \ln{{(}\bm B^\top_{a_{\tau}} \bm s_{\tau+1}^{\pi}{)}}
    + \ln{{(}\bm A^\top \bm o_\tau{)}})&
\end{align}
Note that the softmax function is insensitive to the constant $\bm 1$. Also, for $\tau=1$ the term $\ln{{(}\bm B_{a_{\tau-1}}\bm s_{\tau-1}^{\pi}}{)}$ is replaced by $\bm D$. Finally, $\ln{{(}\bm A^\top \bm o_\tau{)}}$ contributes only to past and present time steps, so for this term is null for $t<\tau\leq T$ since those observations are still to be received. 
\section{Expected Free-energy}
\label{sec:AppendixG}
We indicate with  $G(\pi)$ the expected free-energy obtained over future time steps until the time horizon $T$ while following a plan $\pi$. Basically, this is the variational free-energy of future trajectories which measures the plausibility of plans according to future predicted observations \cite{sajid2021active}. To compute it we take the expectation of variational free-energy under the posterior predictive distribution $P(o_\tau|s_\tau)$. Following \cite{sajid2021active} we can write: 
\begin{equation}
    G(\pi) = \sum_{\tau=t+1}^T G(\pi,\tau)
\end{equation}
then:
\begin{eqnarray}
\nonumber
    G(\pi,\tau) = \mathbb{E}_{\tilde{Q}}\big[\ln{Q(s_\tau|\pi)}- \ln{P(o_\tau,s_\tau|s_{\tau-1})}\big]\\
   = \mathbb{E}_{\tilde{Q}}\big[\ln{Q(s_\tau|\pi)}- \ln{P(s_\tau|o_\tau,s_{\tau-1})} - \ln{P(o_\tau)} \big]
\end{eqnarray}
where $\tilde{Q} = P(o_\tau|s_\tau)Q(s_\tau|\pi)$.
The expected free-energy is:
\begin{equation}
    G(\pi,\tau) \geq  \mathbb{E}_{\tilde{Q}}\big[\ln{Q(s_\tau|\pi)}- \ln{Q(s_\tau|o_\tau,s_{\tau-1}, \pi)} - \ln{P(o_\tau)} \big]
\end{equation}
Equivalently, we can express the expected free-energy in terms of preferred observations \cite{DaCosta2020}:
\begin{equation}
    G(\pi,\tau) =  \mathbb{E}_{\tilde{Q}}\big[\ln{Q(o_\tau|\pi)}- \ln{Q(o_\tau|s_\tau, s_{\tau-1}, \pi)-\ln{P(o_\tau)}}\big]
\end{equation}
Making use of $Q(o_\tau|s_\tau,\pi)=P(o_\tau|s_\tau)$ since the predicted observations in the future are only based on $\bm A$ which is plan independent given $s_\tau$, we have:
\begin{equation}
\label{eq:app_G}
    G(\pi,\tau) = \underbrace{D_{KL}\left[Q(o_\tau|\pi)||P(o_\tau)\right]}_{Expected\  cost}+\underbrace{\mathbb{E}_{Q(s_\tau|\pi)}\left[H(P(o_\tau|s_\tau))\right]}_{Entropy}
\end{equation}
were $H[P(o_\tau|s_\tau)]=\mathbb{E}_{P(o_\tau|s_\tau)}\left[-\ln{P(o_\tau|s_\tau)}\right]$ is the entropy. We are now ready to express the expected free-energy in matrix form, such that we can compute it. From the previous equation, one can notice that plan selection aims at minimizing the expected cost and ambiguity. The latter relates to the uncertainty about future observations given hidden states. In a sense, plans tend to bring the agent to future states that generate unambiguous information over states. On the other hand, the cost is the difference between predicted and prior beliefs about final states. Plans are more likely if they minimize cost, and lead to observations that match prior desires. Minimizing $G$ leads to both exploitative (cost minimizing) and explorative (ambiguity minimizing) behavior. This results in a balance between goal-oriented and novelty-seeking behaviors 

Substituting the sufficient statistics in equation \eqref{eq:app_G}, and recalling that the generative model specifies $P(o_\tau) = \bm C$, {one obtains \cite{smith2021step}:}
\begin{equation}
{G(\pi,\tau) = \underbrace{\bm o_\tau^{\pi\top} \left[\ln{\bm o_\tau^\pi}-\ln\bm C\right]}_{Reward\ seeking} \underbrace{- diag(\bm A ^\top \ln\bm A)^\top\bm s_\tau^\pi}_{Information\ seeking}}
\end{equation}
Note that prior preferences are passed through the softmax function before computing the logarithm.
\section{Updating plan distribution}
\label{sec:AppendixPi}
The update rule for the distribution over possible plans follows directly from the variational free-energy:
\begin{equation}
F\left[\cdot\right] = D_{KL}\left[Q(\pi)||P(\pi)\right] +  \bm \pi ^\top \bm F_\pi
\end{equation}
The first term of the equation above can be written as:
\begin{equation}
    D_{KL}\left[Q(\pi)||P(\pi)\right] = \mathbb{E}_{Q(\pi)}\left[\ln{Q(\pi)}-\ln{P(\pi)}\right]
\end{equation}
Recalling that the approximate posterior over policies is a softmax function of the expected free-energy $Q(\pi)=\sigma(-G(\pi))$ \cite{DaCosta2020,friston2017active}, and taking the gradient with respect to $\bm \pi$ it results:
\begin{equation}
    \frac{\partial F}{\partial \bm\pi} = \ln{\bm \pi} + \bm G_\pi + \bm F_\pi + \bm 1
\end{equation}
{where $\bm G_\pi = (G(\pi_1),G(\pi_2),...)^\top$}
Finally, setting the gradient to zero and normalizing through softmax, the posterior distribution over plans is obtained:
\begin{equation}
    \bm \pi = \sigma(-\bm G_\pi - \bm F_\pi)
\end{equation}
The plan that an agent should pursue is the most likely one.  

%% file: tex/bibio.tex
\vspace{-15mm}
\begin{IEEEbiography}[{\includegraphics[width=1in,height=1.25in,clip,keepaspectratio]{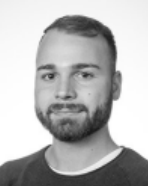}}]{Corrado Pezzato}
received the B.Sc. degree (with Hons.) in Automation Engineering at the Alma Mater Studiorum, Bologna, Italy, in 2017. He received his M.Sc. in Systems and Control  (with Hons.) at the Delft University of Technology in 2019, where he is currently working towards his Ph.D. degree in robotics. He is part of AIRLab, the AI for Retail Lab in Delft. His research interests include low-level control, high-level decision making, and their interconnection, with a strong focus on robotics and active inference.
\end{IEEEbiography}
\vspace{-15mm}

\begin{IEEEbiography}[{\includegraphics[width=1in,height=1.25in,clip,keepaspectratio]{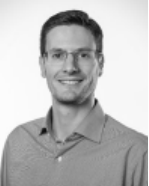}}]{Carlos Hern\'andez Corbato}
received the Graduate degree (with Hons.) in industrial engineering (2006) and the M.Sc. in 2008, and Ph.D. in 2013, in automation and robotics all from the Universidad Polit\'ecnica de Madrid, Madrid, Spain.
He is Assistant Professor in Cognitive Robotics at Delft University of Technology. He has coordinated and participated in European projects on cognitive robotics and factories of the future. His research interests include self-adaptive systems, knowledge representation and reasoning, model-based system engineering.
\end{IEEEbiography}
\vspace{-15mm}

\begin{IEEEbiography}[{\includegraphics[width=1in,height=1.25in,clip,keepaspectratio]{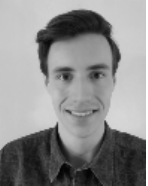}}]{Stefan Bonhof}
is born in Rotterdam, The Netherlands in 1996. He received his B.Sc. in Mechanical Engineering in 2018 and his M.Sc. (with Hons.) in Vehicle Engineering (specializing in autonomy) in 2020, both from the Delft University of Technology. He is currently team lead for students following the MSc Robotics program doing their thesis at AIRLab Delft, as well as technical support and robotics engineer for the whole lab.
\end{IEEEbiography}
\vspace{-12mm}

\begin{IEEEbiography}[{\includegraphics[width=1in,height=1.25in,clip,keepaspectratio]{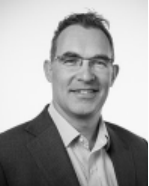}}]{Martijn Wisse}
received the M.Sc. and Ph.D.
degrees in mechanical engineering from the
Delft University of Technology, Delft, The
Netherlands, in 2000 and 2004, respectively.
He is currently a Professor at the Delft University of Technology. His previous research focused on passive dynamic walking robots and passive stability in the field of robot
manipulators. He worked on underactuated
grasping, open-loop stable manipulator control, design of robotic systems, and the creation of startups in this field. His current research interests focus on the neuroscientific principle of active inference and its application and advancements in robotics. 
\end{IEEEbiography}

%% file: root.bbl
\begin{thebibliography}{10}
\providecommand{\url}[1]{#1}
\csname url@samestyle\endcsname
\providecommand{\newblock}{\relax}
\providecommand{\bibinfo}[2]{#2}
\providecommand{\BIBentrySTDinterwordspacing}{\spaceskip=0pt\relax}
\providecommand{\BIBentryALTinterwordstretchfactor}{4}
\providecommand{\BIBentryALTinterwordspacing}{\spaceskip=\fontdimen2\font plus
\BIBentryALTinterwordstretchfactor\fontdimen3\font minus
  \fontdimen4\font\relax}
\providecommand{\BIBforeignlanguage}[2]{{%
\expandafter\ifx\csname l@#1\endcsname\relax
\typeout{** WARNING: IEEEtran.bst: No hyphenation pattern has been}%
\typeout{** loaded for the language `#1'. Using the pattern for}%
\typeout{** the default language instead.}%
\else
\language=\csname l@#1\endcsname
\fi
#2}}
\providecommand{\BIBdecl}{\relax}
\BIBdecl

\bibitem{ghallab2014}
M.~Ghallab, D.~Nau, and P.~Traverso, ``The actor's view of automated planning
  and acting: A position paper,'' \emph{Artificial Intelligence}, vol. 208, pp.
  1 -- 17, 2014.

\bibitem{nau2015}
D.~S. Nau, M.~Ghallab, and P.~Traverso, ``Blended planning and acting:
  Preliminary approach, research challenges,'' in \emph{Twenty-Ninth AAAI
  Conference on Artificial Intelligence}, 2015.

\bibitem{Colledanchise2017}
M.~Colledanchise and P.~Ogren, ``{How Behavior Trees Modularize Hybrid Control
  Systems and Generalize Sequential Behavior Compositions, the Subsumption
  Architecture, and Decision Trees},'' \emph{IEEE Transactions on Robotics},
  vol.~33, no.~2, pp. 372--389, 2017.

\bibitem{Colledanchise2019}
M.~Colledanchise, D.~Almeida, M, and P.~Ögren, ``Towards blended reactive
  planning and acting using behavior tree,'' \emph{in IEEE International
  Conference on Robotics and Automation (ICRA)}, 2019.

\bibitem{safronov2020task}
E.~Safronov, M.~Colledanchise, and L.~Natale, ``Task planning with belief
  behavior trees,'' \emph{IEEE/RSJ International Conference on Intelligent
  Robots and Systems (IROS)}, 2020.

\bibitem{Paxton2019}
C.~{Paxton}, N.~{Ratliff}, C.~{Eppner}, and D.~{Fox}, ``Representing robot task
  plans as robust logical-dynamical systems,'' in \emph{2019 IEEE/RSJ
  International Conference on Intelligent Robots and Systems (IROS)}, 2019, pp.
  5588--5595.

\bibitem{Garrett2020}
C.~R. Garrett, C.~Paxton, T.~Lozano-P{\'e}rez, L.~P. Kaelbling, and D.~Fox,
  ``Online replanning in belief space for partially observable task and motion
  problems,'' in \emph{2020 IEEE International Conference on Robotics and
  Automation (ICRA)}.\hskip 1em plus 0.5em minus 0.4em\relax IEEE, 2020, pp.
  5678--5684.

\bibitem{anil_colored_noise}
A.~Meera and M.~Wisse, ``Free energy principle based state and input observer
  design for linear systems with colored noise,'' in \emph{2020 American
  Control Conference (ACC)}, 2020, pp. 5052--5058.

\bibitem{baioumy2020active}
M.~Baioumy, P.~Duckworth, B.~Lacerda, and N.~Hawes, ``Active inference for
  integrated state-estimation, control, and learning,'' in \emph{International
  conference on Robotics and Automation, ICRA}, 2021.

\bibitem{pezzato2020active}
C.~Pezzato, M.~Baioumy, C.~H. Corbato, N.~Hawes, M.~Wisse, and R.~Ferrari,
  ``Active inference for fault tolerant control of robot manipulators with
  sensory faults,'' in \emph{International Workshop on Active Inference}.\hskip
  1em plus 0.5em minus 0.4em\relax Springer, 2020, pp. 20--27.

\bibitem{baioumy2021fault}
M.~Baioumy, C.~Pezzato, R.~Ferrari, C.~H. Corbato, and N.~Hawes,
  ``Fault-tolerant control of robot manipulators with sensory faults using
  unbiased active inference,'' \emph{European Control Conference, ECC}, 2021.

\bibitem{pezzato2020novel}
C.~Pezzato, R.~Ferrari, and C.~H. Corbato, ``A novel adaptive controller for
  robot manipulators based on active inference,'' \emph{IEEE Robotics and
  Automation Letters}, vol.~5, no.~2, pp. 2973--2980, 2020.

\bibitem{oliver}
G.~Oliver, P.~Lanillos, and G.~Cheng, ``An empirical study of active inference
  on a humanoid robot,'' \emph{IEEE Transactions on Cognitive and Developmental
  Systems}, 2021.

\bibitem{friston2}
K.~J. Friston, ``The free-energy principle: a unified brain theory?''
  \emph{Nature Reviews Neuroscience}, vol. 11(2), pp. 27--138, 2010.

\bibitem{buckley}
C.~Buckley, C.~Kim, S.~McGregor, and A.~Seth, ``The free energy principle for
  action and perception: A mathematical review,'' \emph{Journal of Mathematical
  Psychology}, vol.~81, pp. 55--79, 2017.

\bibitem{tutorial}
R.~Bogacz, ``A tutorial on the free-energy framework for modelling perception
  and learning,'' \emph{Journal of mathematical psychology}, 2015.

\bibitem{friston1}
K.~J. Friston, J.~Mattout, and J.~Kilner, ``Action understanding and active
  inference,'' \emph{Biological cybernetics}, vol. 104(1-2), 2011.

\bibitem{friston3}
K.~J. Friston, J.~Daunizeau, and S.~Kiebel, ``Action and behavior: a
  free-energy formulation,'' \emph{Biological cybernetics}, vol. 102(3), 2010.

\bibitem{friston2012active}
K.~Friston, S.~Samothrakis, and R.~Montague, ``Active inference and agency:
  optimal control without cost functions,'' \emph{Biological cybernetics}, vol.
  106, no. 8-9, pp. 523--541, 2012.

\bibitem{friston2017active}
K.~Friston, T.~FitzGerald, F.~Rigoli, P.~Schwartenbeck, and G.~Pezzulo,
  ``Active inference: a process theory,'' \emph{Neural computation}, vol.~29,
  no.~1, pp. 1--49, 2017.

\bibitem{sajid2021active}
N.~Sajid, P.~J. Ball, T.~Parr, and K.~J. Friston, ``Active inference:
  demystified and compared,'' \emph{Neural Computation}, vol.~33, no.~3, pp.
  674--712, 2021.

\bibitem{schwartenbeck2019computational}
P.~Schwartenbeck, J.~Passecker, T.~U. Hauser, T.~H. FitzGerald, M.~Kronbichler,
  and K.~J. Friston, ``Computational mechanisms of curiosity and goal-directed
  exploration,'' \emph{Elife}, vol.~8, p. e41703, 2019.

\bibitem{Kaplan2018}
R.~Kaplan and K.~J. Friston, ``Planning and navigation as active inference,''
  \emph{Biological cybernetics}, vol. 112, no.~4, pp. 323--343, 2018.

\bibitem{Colledanchise2018}
M.~Colledanchise and P.~Ögren, \emph{Behavior trees in robotics and AI: an
  introduction}.\hskip 1em plus 0.5em minus 0.4em\relax ser. Chapman and
  Hall/CRC Artificial Intelligence and Robotics Series. CRC Press, Taylor \&
  Francis Group, 2018.

\bibitem{macenski2020marathon}
S.~Macenski, F.~Mart{\'\i}n, R.~White, and J.~G. Clavero, ``The marathon 2: A
  navigation system,'' in \emph{2020 IEEE/RSJ International Conference on
  Intelligent Robots and Systems (IROS)}.\hskip 1em plus 0.5em minus
  0.4em\relax IEEE, 2020, pp. 2718--2725.

\bibitem{Orkin2003}
J.~Orkin, ``Applying goal-oriented action planning to games,'' \emph{AI Game
  Programming Wisdom}, vol.~2, pp. 217--228, 2003.

\bibitem{Orkin2006}
------, ``{Three states and a plan: the AI of FEAR},'' in \emph{Game Developers
  Conference}, 2006, pp. 1--18.

\bibitem{Kaelbling2011}
L.~P. Kaelbling and T.~Lozano-P{\'{e}}rez, ``{Hierarchical task and motion
  planning in the now},'' \emph{Proceedings - IEEE International Conference on
  Robotics and Automation}, pp. 1470--1477, 2011.

\bibitem{PackKaelbling2013}
L.~{Pack Kaelbling} and T.~{Lozano-P{\'{e}}rez}, ``{Integrated task and motion
  planning in belief space},'' \emph{International Journal of Robotics
  Research}, pp. 1--60, 2013.

\bibitem{Levihn2013}
M.~Levihn, L.~P. Kaelbling, T.~Lozano-P{\'{e}}rez, and M.~Stilman, ``{Foresight
  and reconsideration in hierarchical planning and execution},'' in \emph{IEEE
  International Conference on Intelligent Robots and Systems}, 2013, pp.
  224--231.

\bibitem{Erol1994}
K.~Erol, J.~Hendler, and D.~S. Nau, ``Htn planning: Complexity and
  expressivity,'' in \emph{AAAI}, vol.~94, 1994, pp. 1123--1128.

\bibitem{Ghallab2016}
M.~Ghallab, D.~Nau, and P.~Traverso, \emph{Automated planning and
  acting}.\hskip 1em plus 0.5em minus 0.4em\relax Cambridge University Press,
  2016.

\bibitem{DaCosta2020}
L.~Da~Costa, T.~Parr, N.~Sajid, S.~Veselic, V.~Neacsu, and K.~Friston, ``Active
  inference on discrete state-spaces: a synthesis,'' \emph{Journal of
  Mathematical Psychology}, vol.~99, p. 102447, 2020.

\bibitem{smith2021step}
\BIBentryALTinterwordspacing
R.~Smith, K.~J. Friston, and C.~J. Whyte, ``A step-by-step tutorial on active
  inference and its application to empirical data,'' \emph{Journal of
  Mathematical Psychology}, vol. 107, p. 102632, 2022. [Online]. Available:
  \url{https://www.sciencedirect.com/science/article/pii/S0022249621000973}
\BIBentrySTDinterwordspacing

\bibitem{hesp2021}
C.~Hesp, R.~Smith, T.~Parr, M.~Allen, K.~J. Friston, and M.~J. Ramstead,
  ``Deeply felt affect: The emergence of valence in deep active inference,''
  \emph{Neural computation}, vol.~33, no.~2, pp. 398--446, 2021.

\bibitem{pymdp}
C.~Heins, B.~Millidge, D.~Demekas, B.~Klein, K.~Friston, I.~Couzin, and
  A.~Tschantz, ``pymdp: A python library for active inference in discrete state
  spaces,'' \emph{Journal of Open Source Software}, vol.~7, no.~73, p. 4098,
  2022.

\bibitem{bt_library}
{Davide Faconti}, ``{BehaviorTree.CPP},'' \url{ https://www.behaviortree.dev/}.

\bibitem{friston2016AILearning}
K.~Friston, T.~FitzGerald, F.~Rigoli, P.~Schwartenbeck, J.~O'Doherty, and
  G.~Pezzulo, ``Active inference and learning,'' \emph{Neuroscience \&
  Biobehavioral Reviews}, vol.~68, pp. 862--879, 2016.

\bibitem{Burridge1999}
R.~R. Burridge, A.~A. Rizzi, and D.~E. Koditschek, ``{Sequential composition of
  dynamically dexterous robot behaviors},'' \emph{International Journal of
  Robotics Research}, vol.~18, no.~6, pp. 534--555, 1999.

\bibitem{Najafi2015}
E.~Najafi, R.~Babu{\v{s}}ka, and G.~A. Lopes, ``{An application of sequential
  composition control to cooperative systems},'' in \emph{2015 10th
  International Workshop on Robot Motion and Control, RoMoCo 2015}, 2015, pp.
  15--20.

\bibitem{cashmore2015rosplan}
M.~Cashmore, M.~Fox, D.~Long, D.~Magazzeni, B.~Ridder, A.~Carrera,
  N.~Palomeras, N.~Hurtos, and M.~Carreras, ``Rosplan: Planning in the robot
  operating system,'' in \emph{Proceedings of the International Conference on
  Automated Planning and Scheduling}, vol.~25, no.~1, 2015.

\bibitem{martin2021}
F.~Mart{\'\i}n, M.~Morelli, H.~Espinoza, F.~J. Lera, and V.~Matell{\'a}n,
  ``Optimized execution of pddl plans using behavior trees,'' in \emph{20th
  International Conference on Autonomous Agents and MultiAgent Systems
  (AAMAS)}, 2021, pp. 1596--1598.

\bibitem{moSSP}
M.~Baioumy, B.~Lacerda, P.~Duckworth, and N.~Hawes, ``On solving a stochastic
  shortest-path markov decision process as probabilistic inference,'' in
  \emph{2nd International Workshop on Active Inference (IWAI)}, 2021.

\bibitem{paquet2005online}
S.~Paquet, L.~Tobin, and B.~Chaib-Draa, ``An online pomdp algorithm for complex
  multiagent environments,'' in \emph{Proceedings of the fourth international
  joint conference on Autonomous agents and multiagent systems}, 2005, pp.
  970--977.

\bibitem{ye2017despot}
N.~Ye, A.~Somani, D.~Hsu, and W.~S. Lee, ``Despot: Online pomdp planning with
  regularization,'' \emph{Journal of Artificial Intelligence Research},
  vol.~58, pp. 231--266, 2017.

\bibitem{baioumy2021towards}
M.~Baioumy, C.~Pezzato, C.~H. Corbato, N.~Hawes, and R.~Ferrari, ``Towards
  stochastic fault-tolerant control using precision learning and active
  inference,'' in \emph{International Workshop on Active Inference}.\hskip 1em
  plus 0.5em minus 0.4em\relax Springer, 2021.

\bibitem{schwobel2018active}
S.~Schw{\"o}bel, S.~Kiebel, and D.~Markovi{\'c}, ``Active inference, belief
  propagation, and the bethe approximation,'' \emph{Neural computation},
  vol.~30, no.~9, pp. 2530--2567, 2018.

\end{thebibliography}
